\documentclass[letterpaper]{article} 
\usepackage{aaai2026}  
\usepackage{times}  
\usepackage{helvet}  
\usepackage{courier}  
\usepackage[hyphens]{url}  
\usepackage{graphicx} 
\usepackage{natbib}  
\usepackage{caption} 
\usepackage{algorithm}
\usepackage{algorithmic}

\usepackage{booktabs}       
\usepackage{amsmath,amsfonts,amsthm}
\usepackage{nicefrac}       
\usepackage{amsmath}
\newtheorem{theorem}{Theorem}
\usepackage{makecell}
\usepackage{subcaption}
\usepackage{rotating}
\usepackage{multirow}
\DeclareCaptionFont{sevenpt}{\fontsize{7pt}{8pt}\selectfont}
\usepackage{booktabs}       
\usepackage{amsmath,amsfonts,amsthm}
\usepackage{nicefrac}       
\usepackage{amsmath}
\usepackage{makecell}
\usepackage{subcaption}
\usepackage{rotating}
\usepackage{needspace}
\usepackage{placeins}
\usepackage{afterpage}
\usepackage{multicol}
\usepackage{caption}
\usepackage{subcaption}
\usepackage{float}
\usepackage{multirow}
\usepackage{newfloat}
\usepackage{listings}
\DeclareCaptionStyle{ruled}{labelfont=normalfont,labelsep=colon,strut=off} 
\floatstyle{ruled}
\newfloat{listing}{tb}{lst}{}
\floatname{listing}{Listing}
\title{Efficient Reinforcement Learning for Zero-Shot Coordination in Evolving Games}

\author{
Bingyu Hui,
Lebin Yu,
Quanming Yao,
Yunpeng Qu,
Xudong Zhang,
Jian Wang\thanks{Corresponding author}
}
\affiliations{
Department of Electronic Engineering, Beijing National Research Center for Information Science and Technology, Tsinghua University, Beijing 100084, China\\
huiby23@mails.tsinghua.edu.cn, paladinee15@gmail.com, qyaoaa@tsinghua.edu.cn, qyp21@mails.tsinghua.edu.cn, zhangxd@tsinghua.edu.cn, jian-wang@tsinghua.edu.cn
}

\usepackage{bibentry}

\begin{document}

\maketitle

\begin{abstract}
Zero-shot coordination(ZSC), a key challenge in multi-agent game theory, has become a hot topic in reinforcement learning (RL) research recently, especially in complex evolving games. It focuses on the generalization ability of agents, requiring them to coordinate well with collaborators from a diverse, potentially evolving, pool of partners that are not seen before without any fine-tuning. Population-based training, which approximates such an evolving partner pool, has been proven to provide good zero-shot coordination performance; nevertheless, existing methods are limited by computational resources, mainly focusing on optimizing diversity in small populations while neglecting the potential performance gains from scaling population size. To address this issue, this paper proposes the Scalable Population Training (ScaPT), an efficient RL training framework comprising two key components: a meta-agent that efficiently realizes a population by selectively sharing parameters across agents, and a mutual information regularizer that guarantees population diversity. To empirically validate the effectiveness of ScaPT, this paper evaluates it along with representational frameworks in Hanabi cooperative game and confirms its superiority.
\end{abstract}

\section{Introduction}

As Multi-Agent Reinforcement Learning (MARL) continues to achieve remarkable success in a wide range of complex games and domains, its scalability is increasingly hindered by the greedy growth of computational resource demands, in the scenarios like agent MOBA games \citep{gao2023towards}, robotic soccer \citep{zhu2024dynamic}, and intelligent manufacturing systems \citep{do2025heterogeneous} owing to the demand for training a diverse and evolving set of agents.
These challenges highlight the need for more scalable and resource-efficient MARL frameworks.

We address this problem in the context of Zero-Shot Coordination (ZSC)—a setting where agents must collaborate with unseen partners. This scenario is central to what we term “evolving games”: environments where the pool of potential partners is not fixed, but rather represents a continuously shifting or evolving set of strategies. While self-play \citep{lowe2017multi} is widely used for training cooperation, it often leads to overfitting, limiting agents' ability to generalize. To overcome this, \citet{hu2020other} introduced the ZSC problem, where agents are evaluated on their ability to coordinate without prior exposure to their partners. Population-based training \citep{jaderberg2017population} offers a promising solution by exposing agents to a diverse set of behaviors during training, thus encouraging more generalizable coordination strategies \citep{charakorn2023generating}. In principle, if the training population covers all potential partner strategies, ZSC would be effectively solved. However, approaching this ideal requires training a large and diverse population of agents, which is itself constrained by limited computational resources, hindering the development of efficient and scalable RL methods. 

\begin{figure}[t!]
  \centering
  \includegraphics[width=\linewidth]{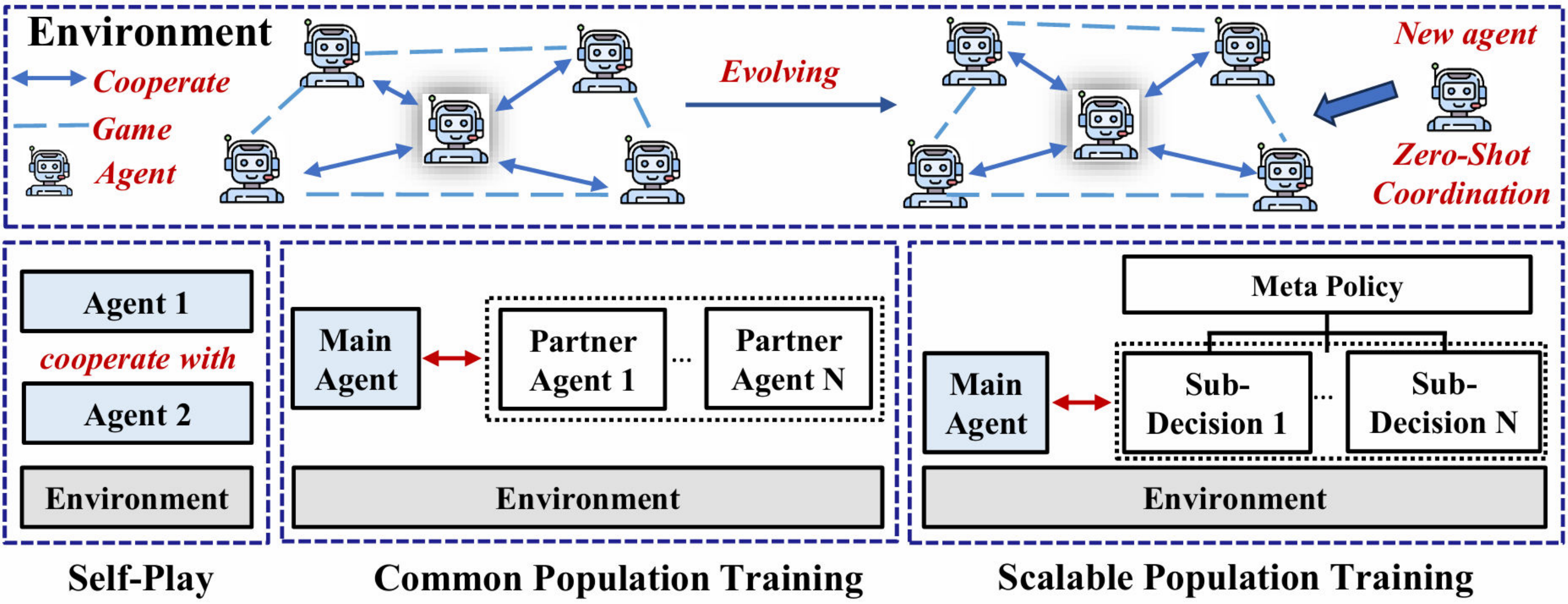}
  \caption{The diagram of different training paradigms in evolving games.}
  \label{fig:structure}
  
\end{figure}

Recent methods focus on improving Zero-Shot Coordination (ZSC) performance by enhancing population diversity, which serves as a key proxy for strategic variation in these evolving games. For example, Trajectory Diversity for Zero-Shot Coordination \citep{lupu2021trajectory} maximizes trajectory-level diversity to foster generalizable strategies. Similarly, Maximum Entropy Population-Based Training \citep{zhao2023maximum} uses entropy maximization to broaden behavioral coverage and improve human-AI coordination. Heterogeneous Multi-Agent Zero-Shot Coordination by Coevolution \citep{xue2024heterogeneous} tackles real-world heterogeneity by co-evolving two distinct populations through iterative pairing and selection, enabling robust cooperation with diverse partners.

However, constrained by limited computational resources, existing studies typically maintain relatively small population sizes and attempt to improve the zero-shot coordination performance of trained agents by optimizing objective functions to encourage policy “diversity” within these small populations. Nevertheless, focusing solely on individual-level behavioral diversity while neglecting the potential impact of population scale may limit the overall effectiveness of population-based methods in truly capturing the breadth of an evolving game's strategy space. In fact, the constrained population size could be one of the critical bottlenecks restricting current ZSC performance.

To address this challenge, we propose an efficient RL with scalable population-based training framework: ScaPT. Figure~\ref{fig:structure} illustrates the differences among self-play, common population-based training, and our proposed framework in evolving games. Our goal is to enable ZSC training to scale more efficiently to larger populations, mitigating the constraints imposed by computational resources. To achieve this, our framework is designed to meet two key criteria:
(1) alleviate the coupling between population size and computational and parameter overhead, thereby mitigating the linear growth of resource consumption with increasing population size;
(2) maintain comparable or even superior ZSC performance relative to conventional population-based methods.
(2) maintain comparable or even superior ZSC performance relative to conventional population-based methods.

To satisfy these requirements, ScaPT employs a hierarchical meta-agent architecture that simulates arbitrarily large populations within a single network, thereby avoiding redundant training across individuals and effectively suppressing the linear growth of computational costs with population size. To promote behavioral diversity, ScaPT maximizes the conditional mutual information between the actions generated by sub-decision modules under given observations and the modules themselves.

In experiments conducted on two standard ZSC benchmark environments, matrix games and Hanabi, ScaPT achieves significantly better performance than existing algorithms, primarily due to its scalable population design that allows it to leverage larger agent pools under the same memory resources constraints. Furthermore, when comparing different algorithms under the same scalable population framework, ScaPT demonstrates outstanding zero-shot coordination game capabilities.

Our contributions are summarized below:
\begin{itemize}
    \item We introduce ScaPT, a scalable population training framework for zero-shot coordination that surpasses the population size limits of conventional methods.
    \item ScaPT designs the scalable population using a hierarchical meta-agent architecture to simulate large populations, achieving a high degree of decoupling between population size and computational resources, while employing conditional mutual information to ensure diversity within the scalable population.
    \item Under constrained memory budgets, ScaPT outperforms prior methods by leveraging significantly larger agent populations, achieving superior zero-shot coordination performance with the same or lower computational cost.
\end{itemize}

\section{Related Work}

Self-play \citep{yu2022surprising} is widely used in MARL, and can efficiently train cooperative agents, but it typically fails under zero-shot coordination. To address this issue, recent works introduce additional reasoning abilities, such as break symmetries of the task to keep agents from learning specified strategies \citep{hu2020other,treutlein2021new,muglich2022equivariant}, conduct multi-level reasoning for higher-level consensus \citep{cui2021k,hu2021off} and predict partners' actions \citep{lucas2022any, yan2024efficient}. However, these approaches often lead to algorithm-level consensus or overfitting to a limited set of partners.

Population-based training enhances generalization by exposing a main agent to diverse partners. Common approaches include reducing the collaboration scores of different partner agents within the population to make them behave differently \citep{charakorn2023generating,rahman2023generating}, improving trajectory diversity of agents \citep{lupu2021trajectory} and increasing policy entropy of the population \citep{zhao2023maximum}. However, most approaches train distinct networks per agent or require differentiable action distributions \citep{lupu2021trajectory, zhao2023maximum}, limiting efficiency and applicability; In contrast, ScaPT employs a meta-agent for efficient population training, and designs a generic mutual information term to guarantee population divergence. This differs from MI-based skill-discovery methods such as DIAYN \citep{eysenbach2018diversity} and DADS \citep{sharma2019dynamics}, which increase intra-agent skill diversity rather than inter-agent policy diversity.

Notably, the meta-agent in ScaPT is different from the agents in meta-RL \citep{nagabandi2018learning,gupta2018meta}. Traditional meta-RL aims to train agents that can quickly adapt to various tasks, the diversity of which is innate and invariable. In contrast, ScaPT's meta-agent is designed to exhibit various policies, the diversity of which is variable and what we seek to augment. 

\section{Scalable Population Training framework for Zero-Shot Coordination}

We present ScaPT, a scalable population-training framework that achieves diverse zero-shot coordination.

\subsection{Hierarchical Meta-Agent For Scalable Population}

A major drawback of population training is its high resource demand. Inspired by meta-RL and multi-task learning \citep{ma2018modeling,sagirova2025srmt}, task-related parameters can be shared, while behavior-specific submodules are optimized independently.

\begin{figure}[t]
    \centering
    \includegraphics[width=\columnwidth]{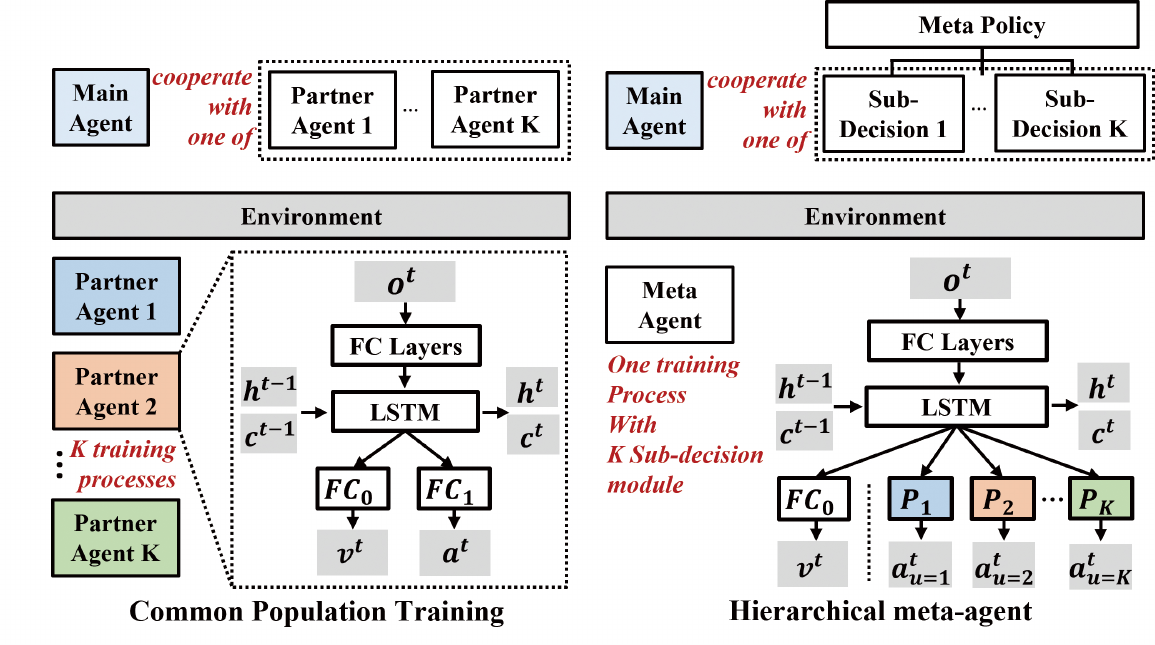}
    \caption{Comparison of common population and hierarchical meta-agent population}
    \label{fig:network}
\end{figure}

Based on the aforementioned idea, we design a hierarchical meta-agent—a simple yet effective architecture—to realize a scalable population. Fig.~\ref{fig:network} show the details of the proposed architecture. Here, $h$ and $c$ denote the hidden and cell states of the LSTM, respectively, and $o^t$ is the observation input from the environment. The variable $v^t$ represents the estimated state value. The submodules $P_{1}, P_{2}, ..., P_{K}$ are used to simulate different individuals in the population, Each of $P_{i}$'s output $a^t_{u=i}$ serves as the policy output for the $i$-th agent in the population, where $K$ denotes the population size.

As can be seen, the hierarchical architecture of the meta-agent greatly reduces the parameters that need to be optimized (K complete sets $\to$ 1 complete set + K subsets). In this way, one can simulate populations of arbitrary size simply by adjusting the number of subsets, without the need to independently train each agent from scratch. In our implementation, we simply define the $P_{i}$ as a two-layer fully connected network; however, this choice is not mandatory. More complex architectures could potentially yield further performance improvements, though this comes with higher computational cost.

\subsection{Conditional Mutual Information Maximization}
If all the agents in a population have similar policies, the best-response agent can only learn to cooperate with partners of one kind of policy, and in this way the advantage of population training disappears (See Ablation study section for details). Therefore, increasing the diversity of population (i.e. different agents in the population behave variously) has always been a key concern in this research area. It is noteworthy that a diverse population means different agents in the population act differently, and this can be achieved by making the meta-agent output distinct actions with different sub-decision modules given an observation and history trajectory. This operation can be formulated as maximizing the conditional mutual information:

\begin{equation}
    I(A;U|H) = \int\int\int p(a,u,h)\log\frac{p(h)p(a,u,h)}{p(u,h)p(a,h)}dudsda
\end{equation}

where $H$ represents observation input (containing current observation and historical trajectory used for decision), $U$ represents the index of sub-decision module that outputs action $A$, and $p(a,u,h)$ is the joint probability density function. Considering that reinforcement learning frameworks commonly estimate integrals using the Monte Carlo method, that is, sampling transitions from replay buffer to calculate, the unbiased estimation of mutual information $\hat{I}(A;U|H)$ can be written as:

\begin{equation}
	\label{eqa:basei}
	\begin{aligned}
		\hat{I}(A;U|H) 
		&= \frac{1}{N}\sum_{j=1}^N \log\frac{p(a_j|u_j,h_j)}{p(a_j|h_j)} \\
		&= \frac{1}{N}\sum_{j=1}^N \bigg[ \log p(a_j|u_j,h_j) \\
		& \qquad - \log \sum_{i=1}^K p(u_i|h_j) \cdot p(a_j|u_i,h_j) \bigg]
	\end{aligned}
\end{equation}
where $K$ is the total number of sub-decision modules, $N$ is the number of transitions, $a_j,u_j,h_j$ are action, sub-decision module index and observation input of the $j$-th transition, and $p(a_j|u_j,h_j)$ is the conditional probability that an agent takes action $a_j$ given $u_j,h_j$. For brevity, use $I_j$ to denote the $j$-th term in $\hat{I}(A;U|H)$. The equation is: $I_j := \log p(a_j|u_j,h_j) -\log \sum_{i=1}^K p(u_i|h_j)p(a_j|u_i,h_j)$.

Notably, directly maximizing $\hat{I}(A;U|H)$ with gradient-based methods is not a preferable choice for two reasons. Firstly, The posterior probability $p(u_i|h_j)$ is hard to calculate. Secondly, the gradient of $p(a_j|u_i,h_j)$ is almost always equal to zero for many RL policies. For example, value-based policies commonly output actions that maximize the action-value function $Q(a,u,h)$ (or use $\epsilon$-greedy for exploration). In this way, the value of $p(a_j|u_i,h_j)$ is determined only by whether $a_j = \arg\max_a Q(a,u_i,h_j)$. Consequently, the derivatives of $p(a_j|u_i,h_j)$ with respect to the neural network parameters are equal to zero. 

We propose to optimize an alternative objective $\bar{I}(A;U|H)$, which have the following two properties:
\begin{enumerate}
    \item $\bar{I}(A;U|H)$ provides gradients that can be used for neural network training, whether the meta-agent makes decisions by outputting action distributions or maximizing Q-functions.
    \item Increasing $\bar{I}(A;U|H)$ also increases $\hat{I}(A;U|H)$.
\end{enumerate}
The definition of $\bar{I}(A;U|H)$ is given below:
\begin{equation}
    \bar{I}(A;U|H) = -\frac{1}{N}\sum_{j=1}^N \sum_{i=1,i\neq j}^K F(u_i,h_j,a_j)
\end{equation}
where $F(u_i,h_j,a_j)$ represents the favor of the meta-agent for $a_j$ given $u_i,h_j$ and is required to be the direct output of a neural network so that its gradients can be used for gradient-based training. Below are several possible forms of $F(u_i,h_j,a_j)$:
\begin{enumerate}
    \item If the neural network used for decision directly outputs action distribution (common for policy-gradient-based methods such as PPO \citep{schulman2017proximal}), then $F(u_i,h_j,a_j)$ represents the probability of choosing $a_j$ given $u_i,h_j$, which is $p(a_j|u_i,h_j)$;
    \item If the neural network used for decision outputs advantage functions (or Q-values, common for value-based methods such as Dueling-DQN \citep{wang2016dueling}), then $F(u_i,h_j,a_j)=A(u_i,h_j,a_j)$ (or $Q(u_i,h_j,a_j)$). 
\end{enumerate}
Moreover, we provide certain theoretical guarantee for maximizing $\bar{I}(A;U|H)$. 

\begin{theorem}
\label{ter:1}
Given $F(u_j,h_j,a_j)$, if $F$ is update to $F'$ such that:

\begin{equation}
\label{eqa:condition}
\begin{split}
    \exists v\ \text{s.t.} \quad 
    & \arg\max_a F(u_v,h_j,a) = a_j \ \land \\
    & F'(u_v,h_j,a_j) < F(u_v,h_j,a_j) \ \land \\
    & \forall i \neq v,\ 
    F'(u_i,h_j,a_j) = F(u_i,h_j,a_j)
\end{split}
\end{equation}

then the corresponding term $I_j$ in $\hat{I}(A;U|H)$ is updated to $I'_j$ and satisfies $I'_j \geq I_j$. 
\end{theorem}
The proof is presented in the Appendix. 

Computing $\bar{I}(A;U|H)$ requires obtaining all agents’ actions under the same observation, which is efficiently achieved via the meta-agent with a single forward pass—unlike standard population training that needs $N$ slow LSTM passes. Notably, ScaPT’s two components are complementary: the meta-agent enables efficient mutual information computation and training, while the conditional mutual information term guides the learning of a diverse population.

\subsection{Instantiation}
ScaPT only specifies how to build an efficient and diverse population, and is compatible with multiple base RL frameworks that optimize agent policies. Our instantiation is based on an value-based approach because it is confirmed that this kind of method is suitable for our experimental task Hanabi \citep{hu2019simplified,bard2020hanabi}. The main agent only needs to cooperate well with all the agents in the population (which are realized using the partner meta-agent), and thus it is required to minimize the base TD-error \citep{van2016deep}: 
\begin{equation}
\label{eqa:mainloss}
\begin{aligned}
L_{m} &= \frac{1}{N}\sum_{j=1}^N r_j + \gamma \max_a Q_{\theta_m'}(h'_j,u_j,a) \\
    & - Q_{\theta_m}(h_j,u_j,a_j)
\end{aligned}
\end{equation}
where $\theta_m$ and $\theta'_m$ represent the parameters in the online Q-net and target Q-net of the main agent respectively. Notably, the main agent has the same neural network architecture as a normal agent shown in Fig.~\ref{fig:network}. In comparison, the meta-agent not only needs to learn coordination, but also needs to become diverse by maximizing $\bar{I}(A;U|H)$. Consequently, $\bar{I}(A;U|H)$ is added to the base TD-loss with a weight $\alpha$, which controls the balance between cooperation ability and population diversity:
\begin{equation}
\label{eqa:final_loss}
\begin{aligned}
L_{p} &= \frac{1}{N}\sum_{j=1}^N \biggl[
r_j + \gamma \max_a Q_{\theta_p'}(u_j,h'_j,a) \\
& - Q_{\theta_p}(u_j,h_j,a_j) + \alpha \sum_{\substack{i=1 \\ i \neq j}}^K Q_{\theta_p}(u_i,h_j,a_j) \biggr]
\end{aligned}
\end{equation}
where $\theta_p$ represent the parameters in the online Q-net of the partner meta-agent. To accelerate convergence, prioritized replay \citep{schaul2015prioritized} and dueling-net \citep{wang2016dueling} is also utilized. 

\begin{table*}[t]
  \centering
  \setlength{\tabcolsep}{2mm}  
  \fontsize{9}{11}\selectfont
  \begin{tabular}{l l l l}
    \toprule
    \textbf{Index} & \textbf{Act Group} & \textbf{Objective for $\pi_m$} & \textbf{Objective for $\pi_{pi}$} \\
    \midrule
    I   & $MP$            & $\sum_{i=1}^N J(\pi_m,\pi_{pi})$                        & $J(\pi_m,\pi_{pi})$ \\
    II  & $MM$, $MP$       & $J(\pi_m,\pi_m) + \sum_{i=1}^N J(\pi_m,\pi_{pi})$       & $J(\pi_m,\pi_{pi})$ \\
    III & $MP$, $PP$       & $\sum_{i=1}^N J(\pi_m,\pi_{pi})$                        & $J(\pi_{pi},\pi_{pi})$ \\
    IV  & $MM$, $MP$, $PP$ & $J(\pi_m,\pi_m) + \sum_{i=1}^N J(\pi_m,\pi_{pi})$       & $J(\pi_{pi},\pi_{pi})$ \\
    V   & $MP$, $PP$       & $\sum_{i=1}^N J(\pi_m,\pi_{pi})$                        & $J(\pi_{pi},\pi_{pi}) + J(\pi_m,\pi_{pi})$ \\
    VI  & $MM$, $MP$, $PP$ & $J(\pi_m,\pi_m) + \sum_{i=1}^N J(\pi_m,\pi_{pi})$       & $J(\pi_{pi},\pi_{pi}) + J(\pi_m,\pi_{pi})$ \\
    \bottomrule
  \end{tabular}
  \caption{Feasible population training modes}
  \label{tab:1}
\end{table*}

Another key component of instantiation is the training mode, which specifies the pairs of agents interacting with the environment and the corresponding transitions for training. Table~\ref{tab:1} summarizes six feasible modes, with details in the Appendix. Take Mode-III as an example: it has act groups $MP,PP$, meaning agents generate transitions in two groups: [main agent, partner agent] and [partner agent, partner agent]. The training objectives require the main agent to cooperate well with the partner, while the partner only optimizes self-play scores without adapting. Each mode emphasizes different aspects, and we investigate their performance experimentally.

\section{Experiments}

\subsection{Matrix game: Task Complexity vs. Population Size}

To investigate how population size influences the ZSC performance of the main agent, we conducted comparative evaluations using collaborative single-step matrix games \citep{lupu2021trajectory}. In each game, player 1 selects a row and player 2 selects a column independently. After both players have made their choices, the selected actions are revealed, and the agents receive the reward corresponding to the intersection of the chosen row and column. Specifically, we compared our approach against representative baselines, including \textbf{ET3} \citep{yan2024efficient}, \textbf{MEP} \citep{zhao2023maximum}, and \textbf{TrajeDi} \citep{lupu2021trajectory}, \textbf{Individual agents}, where agents are trained independently, and \textbf{baseline} \citep{lupu2021trajectory}. These methods cover both population-based and non-population-based paradigms, providing both a comparison across different algorithmic architectures and an in-depth view of how population size impacts the ZSC performance of population-based methods.

Fig.~\ref{fig:M_50_2} presents the experimental results for various baseline algorithms under a matrix dimension of 50 and a fixed population size of 2. It can be observed that while all methods achieved strong performance in the self-play setting, their scores dropped significantly in the cross-play evaluation, indicating a lack of generalization between agents. Notably, ScaPT retained a non-negligible degree of ZSC capability, outperforming other baselines in this regard.

As illustrated in Fig.~\ref{fig:M_50_30}, when the population size was increased, ScaPT’s performance in cross-play improved markedly, demonstrating a strong resurgence in ZSC ability. This suggests that scaling the population size can enhance the diversity and representational capacity of the training process, thereby equipping the main agent with more generalized coordination strategies. The appendix includes examples of matrices with varying dimensions as well as the remaining experimental results.

\begin{figure}[t]
  \centering
  \begin{subfigure}[t]{1\columnwidth}
    \centering
    \includegraphics[width=\textwidth]{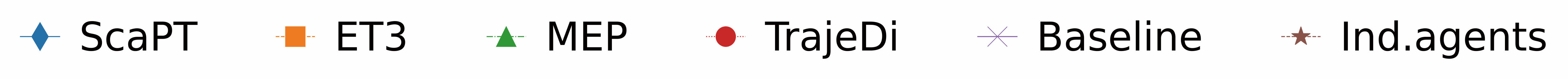}
  \end{subfigure}

  \begin{subfigure}[t]{0.98\columnwidth}
    \centering
    \includegraphics[width=\textwidth]{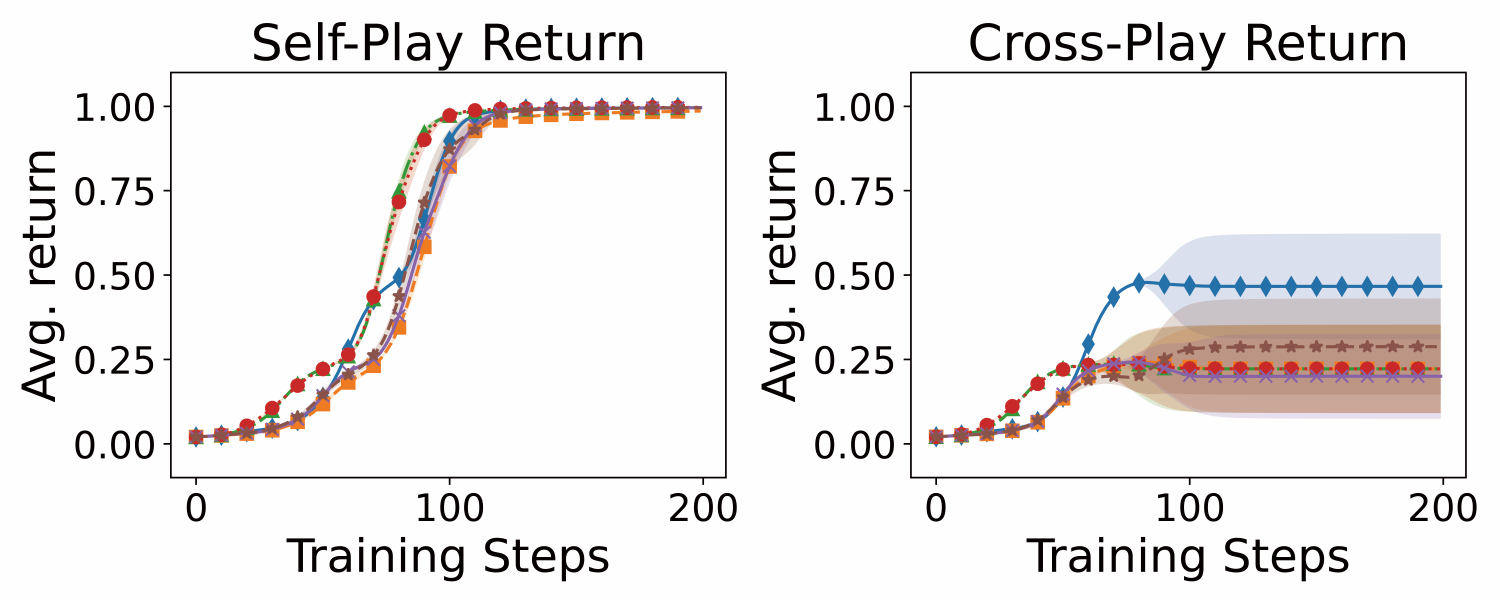}
    \caption{Matrix dimension = 50, population size = 2}
    \label{fig:M_50_2}
  \end{subfigure}
  
  \begin{subfigure}[t]{0.98\columnwidth}
    \centering
    \includegraphics[width=\textwidth]{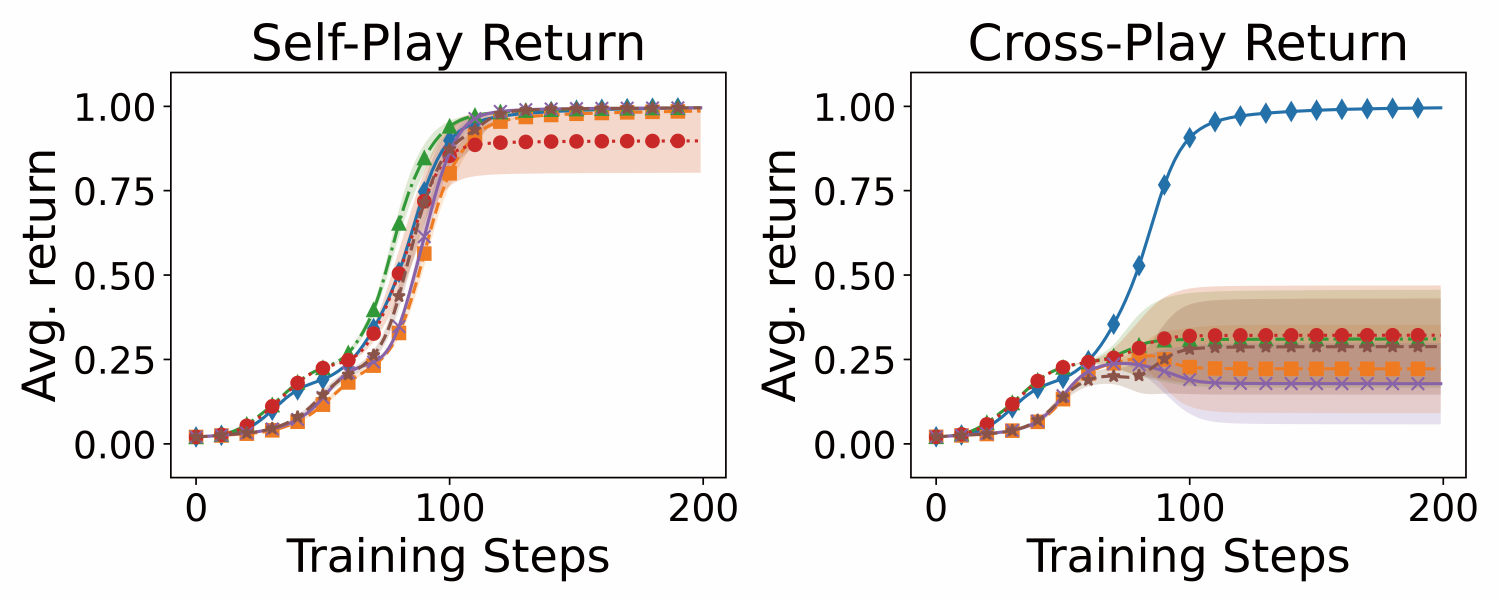}
    \caption{Matrix dimension = 50, population size = 30}
    \label{fig:M_50_30}
  \end{subfigure}
  
  \caption{Results for matrix dimension = 50 with different population sizes}
  \label{fig:M_50}

\end{figure}

\subsection{ScaPT in Hanabi}

We conducted experiments on Hanabi \citep{bard2020hanabi}, a cooperative card game in which players must play cards in order by color and rank. Players can only see their teammates' cards, not their own, and must infer intentions and communicate through limited actions. As a result, self-play agents perform well due to shared training, but collaboration with unfamiliar agents is challenging, making Hanabi a standard benchmark for zero-shot coordination.

\subsection{Compare methods and Evaluation Criteria}
\label{sec:criteria}
We compare our method with five representative methods, including one self-play methods: \textbf{SP} — self-play training with parameter sharing, serving as a baseline; two non-population-based algorithms designed for ZSC: \textbf{OP} \citep{hu2020other}, \textbf{OBL} \citep{hu2021off}; and two population-based ZSC algorithms, \textbf{TrajeDi} \citep{lupu2021trajectory} and \textbf{MEP} \citep{zhao2023maximum}.

Zero-shot coordination requires agents to cooperate well with collaborators that have not been seen before, that is, ``strangers''. The corresponding evaluation criteria are slightly different from those of normal MARL. Below we give a formulaic representation in two-agent scenarios. 

Use $J(\pi_1,\pi_2)$ to denote the expected cumulative discounted return obtained by the collaboration of $\pi_1$ and $\pi_2$. Use $\pi^M_i$ to denote the policy obtained with training framework $M$ and random seed $i$. The earliest metric to evaluate zero-shot cooperation performance is intra-algorithm cross-play (abbreviated as Intra-XP) score \citep{hu2020other}:
\begin{equation}
    S_{intra-XP}(M) = \mathbb{E}[J(\pi^M_i,\pi^M_j)|i\neq j]
\end{equation}
Cross-play is a commonly used metric that evaluates agent's ZSC performance with partner agents trained using the same algorithm but different random seeds. However, since all agents share the same architecture, the test agent may retain implicit prior knowledge about its partners. Thus, relying solely on cross-play does not provide a rigorous assessment of zero-shot coordination.

In order to address the deficiencies of the aforementioned metric, \citet{lucas2022any} propose one-sided zero-shot coordination (abbreviated as 1ZSC-XP) score:
\begin{equation}
    S_{1ZSC-XP}(M) = \mathbb{E}[J(\pi^M,\pi^{M_t})]
\end{equation}
where $M_t$ refers to a set of algorithms that are not specially designed for zero-shot coordination. The shortcoming of this criterion is that $\pi^{M_t}$ still cannot represent all feasible ``strangers'', and the results may be biased. 

Neither of the metrics is perfect, hence the following sections display the above two metrics for more comprehensive evaluation.

To enhance the credibility of the evaluation, we train models with several fixed random seeds(time consumption: 0-4 (2-player), 1-4 (5-player)) under each experiments, and test Intra-XP and 1ZSC-XP scores. Specifically, 1ZSC-XP scores are obtained by pairing the tested zero-shot coordination agents with 40 non-ZSC agents obtained with four kinds of self-play frameworks: IQL \citep{tan1993multi}, VDN \citep{sunehag2018value}, SAD and SAD+AUX \citep{hu2019simplified}. 

\subsection{Experiments with different framework under Resource Constraints}

Firstly, in Hanabi, we evaluate each algorithm within its native framework under a strict resource constraints (128GB RAM and 48GB GPU memory), reflecting typical small-scale research servers. 

We compare our method only with population methods designed for ZSC: TrajeDi and MEP on 5-player game. Table~\ref{tab:tradeoff} shows the performance upper bounds of different methods under constrained RAM and GPU memory. Specifically, under these limits, TrajeDi and MEP can only support a population size of 1 due to the need to allocate separate replay buffers for each agent. In contrast, ScaPT achieves the best performance for both metrics by flexibly scaling the population size to 5 and 8.

\begin{table}[t]
  \centering
  \setlength{\tabcolsep}{2mm}
  \renewcommand{\arraystretch}{1.1}
  \fontsize{9}{11}\selectfont
  \begin{tabular}{lcccc}
    \toprule
    & Metrics & TrajeDi & MEP & ScaPT \\
    \midrule
    \multirow{2}{*}
      & Intra-XP   & 4.17$\pm$1.34 & 5.13$\pm$2.11 & \textbf{11.07$\pm$1.65} \\
      & 1ZSC-XP    & 4.22$\pm$1.47 & 6.69$\pm$1.62 & \textbf{11.39$\pm$1.60} \\
    \midrule
  \end{tabular}
  \caption{Best performance under Resource Constraints in 5-player hanabi game}
  \label{tab:tradeoff}
\end{table}

To validate our archictecture and ensure fair comparison, we evaluate MEP under common and scalable population architectures in computational cost, training time, and ZSC performance on 2-player Hanabi. As shown in Fig.~\ref{fig:comparision}, the scalable population scales more efficiently with slower growth in resource use and training time, while the common framework grows approximately linearly. Both approaches achieve comparable ZSC performance, demonstrating the scalability and effectiveness of our method for complex tasks requiring large population agents.

\begin{table*}[t]
  \centering
  \setlength{\tabcolsep}{2mm} 
  \renewcommand{\arraystretch}{1.1} 
  \fontsize{9}{11}\selectfont 
  \begin{tabular}{lcccccc}
    \toprule
    & SP & OP & OBL & TrajeDi* & MEP* & ScaPT \\
    \midrule
    Intra-XP   & 3.96$\pm$0.49 & 14.94$\pm$0.67 & \textbf{23.80$\pm$0.03} & 15.95$\pm$2.47 & 20.09$\pm$0.69 & 20.77$\pm$0.40 \\
    1ZSC-XP    & 7.68$\pm$0.39 & 13.48$\pm$0.19 & 3.80$\pm$0.07 & 14.13$\pm$3.20 & 15.64$\pm$3.02 & \textbf{15.95$\pm$3.05} \\
    \bottomrule
  \end{tabular}
  \caption{Zero-shot coordination performance of different algorithms on the 2-player Hanabi game (population-based methods: TrajeDi, MEP, ScaPT, with population size 5). Algorithms marked * are implemented under the proposed scalable framework.}
  \label{tab:main}
  
\end{table*}

\begin{figure}[t]
    \centering
    \includegraphics[width=\columnwidth]{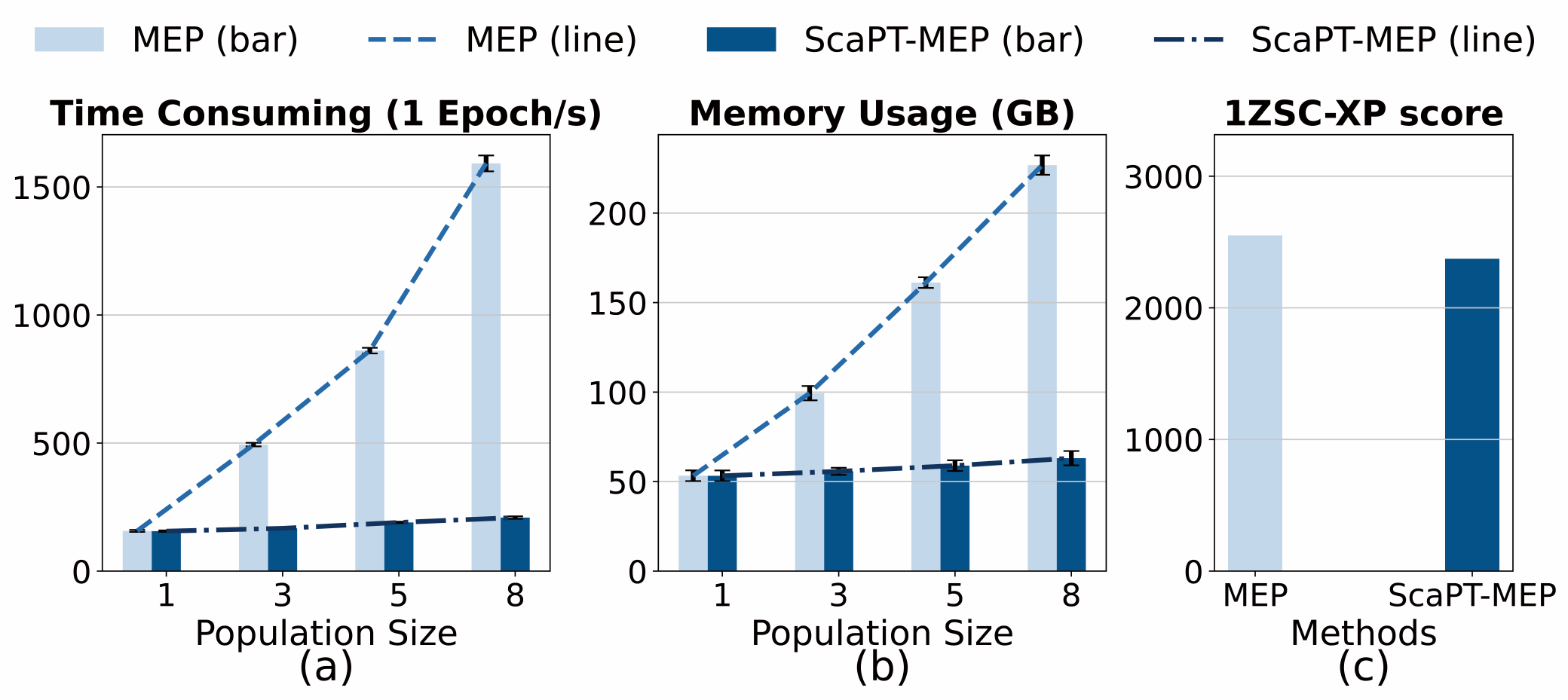}
    \caption{Comparison of common population and meta-agent population: (a) training time consuming with population size increase; (b) resource consumption with population size increase; (c) sum of 1ZSC-XP scores over 40 experiments with different architectures using the MEP method. }
    \label{fig:comparision}
\end{figure}

\begin{figure}[t]
    \centering
    \includegraphics[width=\columnwidth]{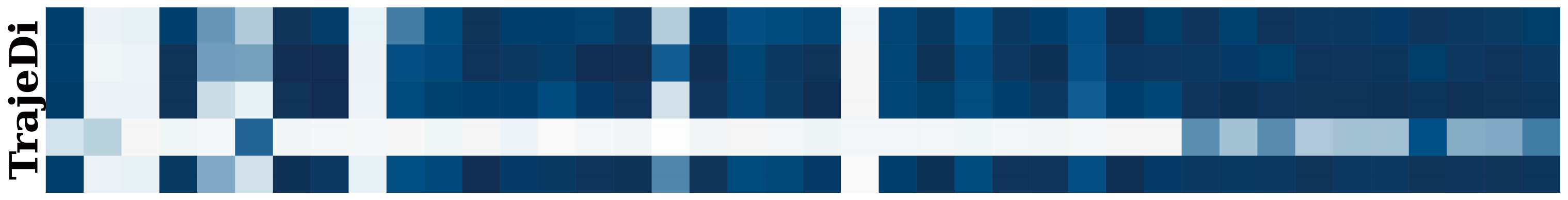}
    \includegraphics[width=\columnwidth]{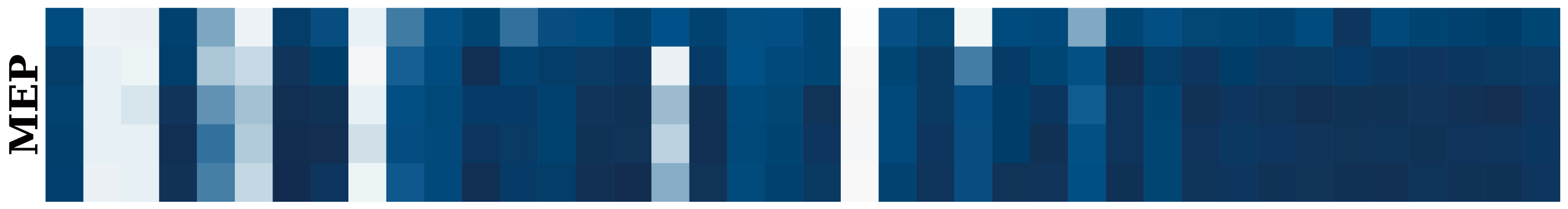}
    \includegraphics[width=\columnwidth]{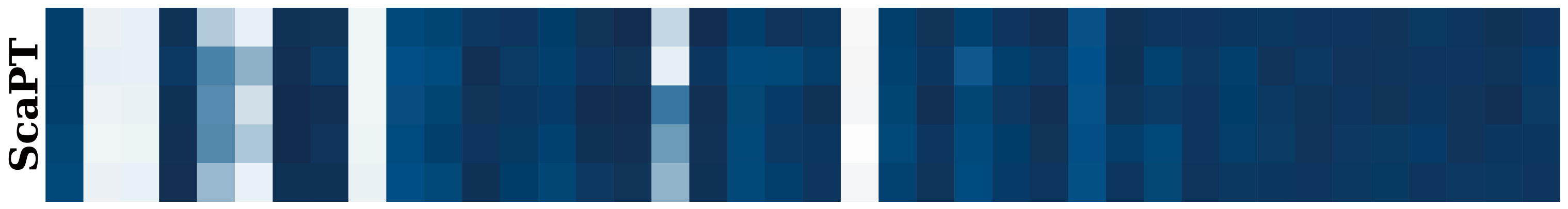}
    \caption{Detailed pair-wise 1ZSC-XP scores of TrajeDi, MEP and CMIMP. Deeper colors represent higher scores and each row represents the coordination scores of testing a main agent pairing with 40 non-ZSC agent, thus forming a $5\times40$ heat-map.}
    \label{fig:detailed}
\end{figure}

\subsection{Experiments with scalable population framework}

We conduct a 2-player Hanabi experiment with population size 5 for all population-based algorithms. Table~\ref{tab:main} reports mean and standard error of evaluation metrics. ScaPT achieves the highest 1ZSC-XP, demonstrating strong coordination with 40 unseen partners and ranks top on Intra-XP except for OBL. Although OBL scores high on Intra-XP, its 1ZSC-XP score is only 3.8, suggesting that OBL lacks genuine zero-shot coordination capability since it performs well mainly when cooperating with agents that are trained using the same algorithm as itself. Considering both metrics together, all population-based methods demonstrate solid performance. Among them, ScaPT delivers the most outstanding results, further confirming its superiority.

Besides, Fig.~\ref{fig:detailed} visualizes the detailed pair-wise 1ZSC-XP cooperation scores of Table~\ref{tab:main} using a heat-map. The heat-maps of three compared population-based methods are presented along for reference, and the heat-maps of other frameworks are presented in the Appendix. In each sub-figure, each row represents the coordination scores of a testing main agent pairing with all non-ZSC agents, and deeper colors represent higher scores. Different rows in a sub-figure correspond to agents trained with different random number seeds under the same framework. There are two phenomena worth noting: Firstly, the differences between columns are consistent for these two frameworks. The reason is that some non-ZSC agents are relatively easy to cooperate with (e.g. column 7,8), while some others are not (e.g. column 2,3). Secondly, the performance stability of ScaPT is good: 1ZSC-XP scores of ScaPT models with different random seeds vary little, while TrajeDi and MEP demonstrate certain instability.

To validate the relationship between population size and ZSC performance, we conducted experiments on the Hanabi benchmark with all the population-based algorithms. By increasing the number of players from 2 to 5, we raised task difficulty to evaluate algorithm performance, while also varying population size to assess its effect on ZSC outcomes.

\begin{table}[t]
  \centering
  \setlength{\tabcolsep}{2mm}
  \renewcommand{\arraystretch}{1.1}
  \fontsize{9}{11}\selectfont
  \begin{tabular}{lcccc}
    \toprule
    ps & Metrics & TrajeDi* & MEP* & ScaPT \\
    \midrule
    \multirow{2}{*}{ps=2} 
      & Intra-XP   & 6.08$\pm$1.86 & 8.13$\pm$2.11 & 7.88$\pm$1.72 \\
      & 1ZSC-XP    & 3.71$\pm$0.60 & 7.08$\pm$1.62 & 6.18$\pm$1.53 \\
    \midrule
    \multirow{2}{*}{ps=5} 
      & Intra-XP   & 9.91$\pm$1.48 & 9.51$\pm$1.64 & \textbf{11.07$\pm$1.65} \\
      & 1ZSC-XP    & 8.38$\pm$0.75 & 8.84$\pm$0.93 & 11.35$\pm$0.58 \\
    \midrule
    \multirow{2}{*}{ps=8} 
      & Intra-XP   & 9.04$\pm$0.44 & 10.98$\pm$0.33 & 10.61$\pm$0.93 \\
      & 1ZSC-XP    & 9.86$\pm$1.73 & 10.48$\pm$1.60 & \textbf{11.39$\pm$1.60} \\
    \midrule
    \multirow{2}{*}{ps=15} 
      & Intra-XP   & 8.98$\pm$0.87 & 6.78$\pm$0.37 & 8.40$\pm$0.83 \\
      & 1ZSC-XP    & 9.01$\pm$1.65 & 8.25$\pm$1.09 & 9.10$\pm$1.81 \\
    \bottomrule
  \end{tabular}
  \caption{ZSC performance of population-based algorithms on 5-player Hanabi, under the scalable framework with different population sizes (PS).}
  \label{tab:psvsper}
\end{table}

\begin{table}[t]
  \centering
  \small  
  \setlength{\tabcolsep}{2mm}  
  \begin{tabular}{c cccccc}
    \toprule
    & I & II & III & IV & V & VI \\
    \midrule
    Intra-XP   & 1.16  & \textbf{20.77} & 8.93  & 12.43 & 12.07 & 19.22 \\
    1ZSC-XP    & 5.63  & \textbf{15.95} & 10.93 & 12.59 & 12.41 & 15.09 \\
    \bottomrule
  \end{tabular}
  \caption{Zero-shot coordination performance under different training modes (mean values only).}
  \label{tab:modes}
\end{table}

As shown in Table~\ref{tab:psvsper}, increasing the number of players from 2 to 5, which raises task complexity, causes all algorithms to experience performance changes. TrajeDi suffers the largest drop, while MEP and ScaPT maintain relatively better performance. Overall, MEP performs strongest, with ScaPT showing comparable but slightly lower scores and more stable results across random seeds. Increasing the population size from 2 to 5 improves performance for all methods considerably—specifically, the 1ZSC-XP metric improves by 3.83 for TrajeDi, 1.38 for MEP, and 3.19 for ScaPT. As the population size increases from 5 to 8, the overall performance of all methods continues to improve, though with reduced marginal gains. However, a further increase to 15 results in performance degradation, indicating the presence of algorithm-specific optimal population sizes—an aspect that merits deeper exploration. ScaPT achieves the best overall results on both metrics, demonstrating strong generalization and underscoring the critical importance of population size for ZSC.

\subsection{Ablation Study}
\label{sec:abl}
As stated in equation~\ref{eqa:final_loss}, the training objective includes a weighted mutual information term that encourages behavioral differences among sub-modules. To examine its impact, Fig. 7 plots four key metrics across training epochs in 2-player hanabi environment under different training modes.

\textbf{MM Score}: The self-play score of the main agent; 

\textbf{MP Score}: The cooperation score of the main agent and the partner agent, one of the key objectives of population training;

\textbf{PP Score}: The self-play score of the partner agents;

\textbf{Diff Prob}: The probability that different agents select the same action given the same observation—lower values indicate higher population diversity.

\begin{figure}[ht]
    \centering
    \begin{minipage}[t]{0.98\columnwidth}
    \centering
    \includegraphics[width=\textwidth]{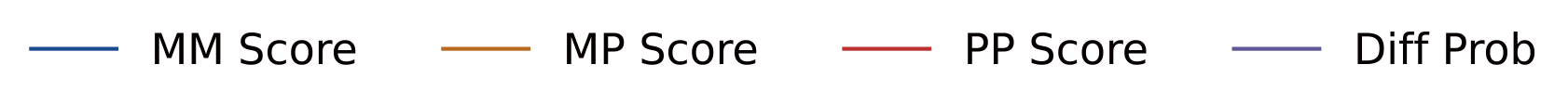}
    \end{minipage}
    
    \begin{minipage}[t]{0.48\columnwidth}
    \centering
    \includegraphics[width=\textwidth]{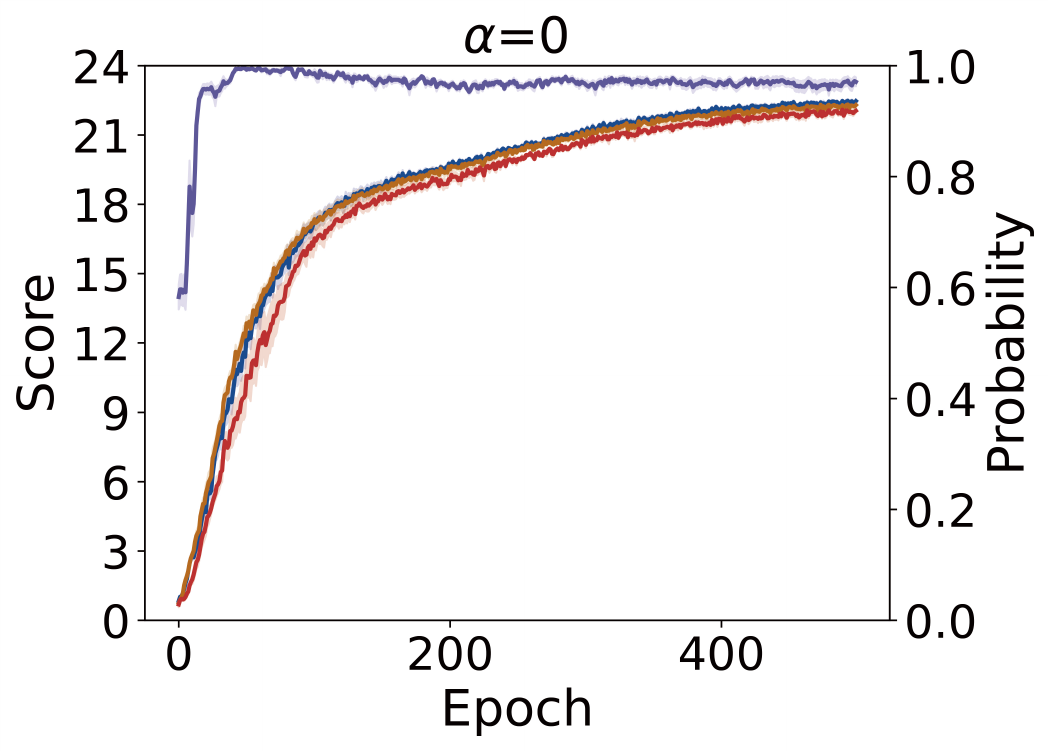}
    \centerline{Intra-XP: 13.44$\pm$2.02}
    \centerline{1ZSC-XP: 13.08$\pm$0.33}
    \end{minipage}
    \hfill
    \begin{minipage}[t]{0.48\columnwidth}
    \centering
    \includegraphics[width=\textwidth]{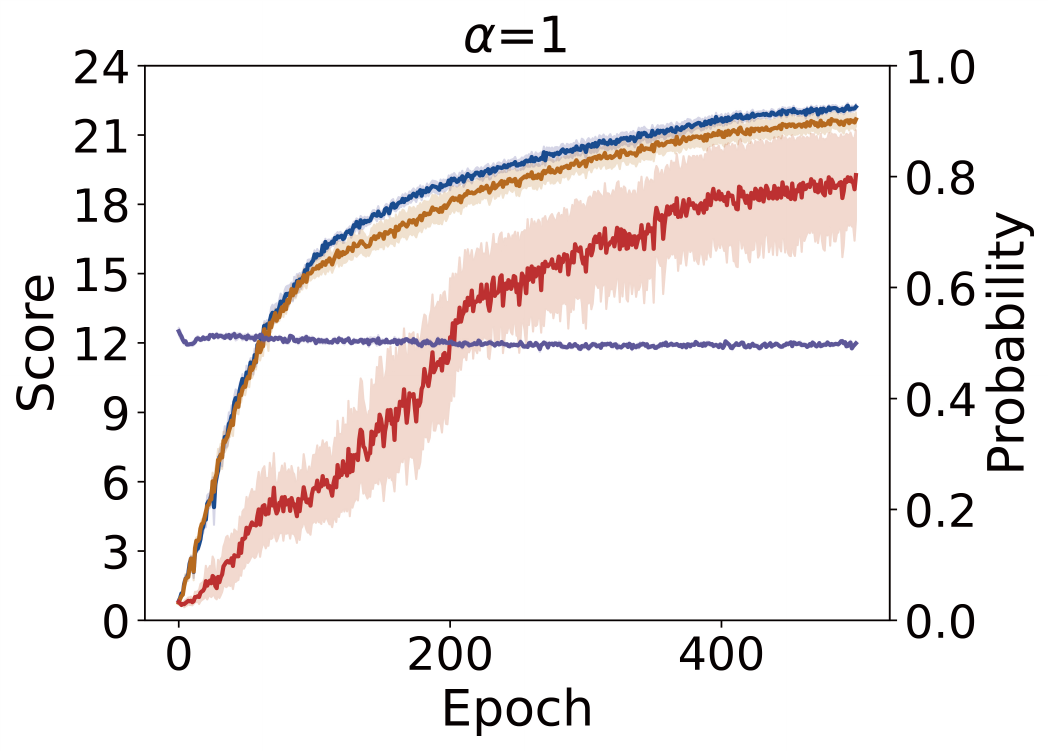}
    \centerline{Intra-XP: 21.05$\pm$0.05}
    \centerline{1ZSC-XP: 15.73$\pm$0.03}
    \end{minipage}

    \begin{minipage}[t]{0.48\columnwidth}
    \centering
    \includegraphics[width=\textwidth]{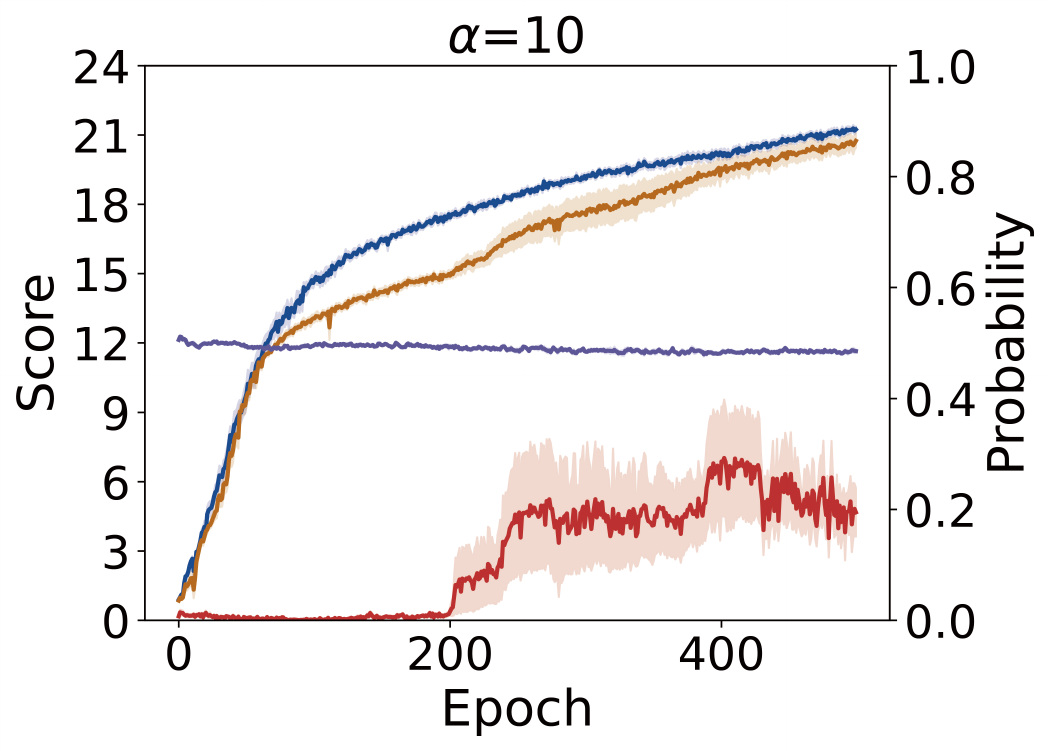}
    \centerline{Intra-XP: 12.88$\pm$1.06}
    \centerline{1ZSC-XP: 10.68$\pm$0.27}
    \end{minipage}

    \caption{Training curves and testing scores of ScaPT with different $\alpha$.}
    \label{fig:weight}
\end{figure}

As is shown in Fig.~\ref{fig:weight}, when the mutual information term $\bar{I}(A;U|H)$ is ignored (i.e. $\alpha=0$), \textbf{Diff Prob} quickly rises to $1$, meaning that the partner agents in the population act similarly. As a result, the generalization performance of the main agent is reduced (low XP scores during testing) despite that the training process goes smoothly (high MM/MP/PP scores during training). When $\alpha$ is set to a proper value ($\alpha=1$),  \textbf{Diff Prob} maintains relatively low, indicating a diverse population. A larger $\alpha$ ($\alpha=10$) does not further increase diversity and reduces partner self-play performance (PP score), suggesting that excessive weight can harm rationality and coordination. To sum up, $\bar{I}(A;U|H)$ can help build a diverse population and thereby improve the zero-shot coordination, while assigning it a too large weight might have negative effect.

\subsection{Comparison of Different Training Modes}
\label{sec:modes_exp}
Table~\ref{tab:1} introduces six feasible training modes for population-based training. Then which mode is the best? Tab.~\ref{tab:modes} confirms that zero-shot coordination performance of different training modes varies a lot, and such differences are brought by the settings of secondary objectives. Mode-I is the worst, indicating that only optimizing the primary objective ($J(\pi_m,\pi_p)$) is not enough. Mode-II is the best, confirming the necessity of adding the self-play objective for the main agent. Notably, Mode-IV and Mode-VI additionally require increasing self-play scores for partner agents on the basis of Mode-II, and this operation is of no benefit judging from the results. The training curves for the different modes are provided in the appendix for reference.

\section{Conclusion}

In this paper, we analyzed limitations of population-based 
RL methods for ZSC, which struggle to scale efficiently enough to represent the diversity of evolving games, and proposed a population-size-free RL training framework. Experiments demonstrated its effectiveness and revealed a strong correlation between population size and ZSC performance, further validating the significance of our study.

\clearpage
\appendix

\begin{center}
    {\LARGE\bfseries Appendix}
\end{center}
\section{A. Proof of Theorem 1}
\renewcommand{\theequation}{A.\arabic{equation}}
\setcounter{equation}{0}
Here, we restate our \textbf{Theorem 1}, where $u_j$, $h_j$ and $a_j$ denote the $j-th$ sub-decision module's sub-decision module index, observation and action, respectively.
\begin{theorem}
\label{ter:1}
Given $F(u_j,h_j,a_j)$, if $F$ is update to $F'$ such that:
\begin{equation}
\label{eqa:condition}
\begin{split}
    \exists v\ \text{s.t.} \quad 
    & \arg\max_a F(u_v,h_j,a) = a_j \ \land \\
    & F'(u_v,h_j,a_j) < F(u_v,h_j,a_j) \ \land \\
    & \forall i \neq v,\ 
    F'(u_i,h_j,a_j) = F(u_i,h_j,a_j)
\end{split}
\end{equation}

then the corresponding term $I_j$ in $\hat{I}(A;U|H)$ is updated to $I'_j$ and satisfies $I'_j \geq I_j$. 
\end{theorem}

\begin{proof}
In consideration of the relationship between $F(u,h,a)$ and $p(u,h,a)$, there are two cases to be addressed.
\paragraph{Case 1}$F(u,h,a) = p(a|u,h)$

With the condition stated in (\ref{eqa:condition}), only $p(a_j|u_v,h_j)$ will be changed among all the terms in $I_j$. Consequently, 

\begin{equation}
\begin{split}
    I'_j - I_j =\ &\log \Big[ 
        p(u_v|h_j)\, p(a_j|u_v,h_j) \\
        &\quad + \sum_{\substack{i=1\\i\neq v}}^K 
        p(u_i|h_j)\, p(a_j|u_i,h_j) 
    \Big] \\
    & - \log \Big[ 
        p(u_v|h_j)\, p'(a_j|u_v,h_j) \\
        &\quad + \sum_{\substack{i=1\\i\neq v}}^K 
        p(u_i|h_j)\, p(a_j|u_i,h_j) 
    \Big]
\end{split}
\end{equation}

Since $p'(a_j|u_v,h_j) < p(a_j|u_v,h_j)$ and $p(u_v|h_j) \geq 0$, $I'_j \geq I_j$.

\paragraph{Case 2}$F(u,h,a) = A(u,h,a)$

In this case, $p(a|u,h) = 1$ if $a = \arg\max Q(u,h,a)$ where $Q(u,h,a)=A(u,h,a)+V(u,h)$, else $p(a|u,h) = 0$ \footnote{If the meta-agent uses an $\epsilon$-greedy strategy for exploration, then the corresponding value is $1-\epsilon$ and $\epsilon/(|A|-1)$. This difference has no impact on the proof.}.

According to (\ref{eqa:condition}), only $Q(u_v,h_j,a_j)$ changes, and this leads to three possible outcomes:

\begin{enumerate}
    \item $a_j \neq \arg\max Q(u_v,h_j,a)$ and \\
    \hspace{1em} $a_j \neq \arg\max Q'(u_v,h_j,a)$. \\
    In this situation, $p(a_j|u_v,h_j)$ remains the same, and $I'_j = I_j$.

    \item $a_j = \arg\max Q(u_v,h_j,a)$ and \\
    \hspace{1em} $a_j = \arg\max Q'(u_v,h_j,a)$. \\
    Similarly, $I'_j = I_j$.

    \item $a_j = \arg\max Q(u_v,h_j,a)$ and \\
    \hspace{1em} $a_j \neq \arg\max Q'(u_v,h_j,a)$. \\
    In this situation, $p(a_j|u_v,h_j)=1$ and $p'(a_j|u_v,h_j)=0$. \\
    As is proved before, $p'(a_j|u_v,h_j) < p(a_j|u_v,h_j)$ leads to $I'_j \geq I_j$.
\end{enumerate}
\end{proof}

\section{B. Implementation Details}
\paragraph{Hardware and software settings} We experiment on a server with 2 x RTX 3090 and a Intel Xeon Platium CPU (12 cores), and training one scalable population models takes around 15 hours. The experimental codes are modified based on the open source codes of OBL. 

\paragraph{Neural network hyper parameters} Fig.2 in the main text shows the architecture of meta agent with K sub-decision modules, which implies that it possesses an equivalent population size of K, and below introduces the hyper parameters of each module. FC layer is a linear transform with output size 512. LSTM has two layers with hidden dim 512. $P_i$ is a two-layer fully
connected network.

\paragraph{Training hyper parameters} All the models are training 500 epochs with replay buffer size 35000 and batch size 128. Parameters are updated via Adam optimizer with learning rate 6.25e-5. Discount factor $\gamma$ is set to 0.999.

\section{C. Detailed results of Matrix game}
This section present all all experiments on Matrix game with different dimensions. To evaluate the performance of various algorithms as task complexity increases, we raise the game difficulty by constructing matrices of varying dimensions, each based on a base 10×10 matrix template. The pseudocode for matrix generation is given below.

\begin{algorithm}[htbp]
\caption{Generate Shifted Block Matrix}
\label{alg:shifted-matrix}
\begin{algorithmic}[1]
\REQUIRE Block count $n_b$, block size $d=10$
\ENSURE Matrix $M \in \mathbb{R}^{(n_b d) \times (n_b d)}$
\STATE Initialize $M \gets 0$
\FOR{$k = 0$ to $n_b - 1$}
    \STATE $B \gets I_d$, $B[1][1] \gets 0$, $B[1][0] \gets 1$
    \FOR{$i = 2$ to $d{-}1$}
        \FOR{$j \in \{i{-}1,\ i{+}1\}$}
            \IF{$1 < j < d$}
                \STATE $B[i][j] \gets \epsilon$
            \ENDIF
        \ENDFOR
    \ENDFOR
    \STATE Flatten $B$ to $b$, compute $s = (\text{arange}(d^2) + k)\bmod d^2$
    \STATE Reshape $b[s]$ to $B' \in \mathbb{R}^{d \times d}$
    \STATE Insert $B'$ into $M[kd{:}(k{+}1)d,\ kd{:}(k{+}1)d]$
\ENDFOR
\RETURN $M$
\end{algorithmic}
\end{algorithm}

\needspace{6\baselineskip}

\subsection{Matrix dimension = 10}
\begin{figure}[H]
    \centering
    \includegraphics[width=0.8\columnwidth]{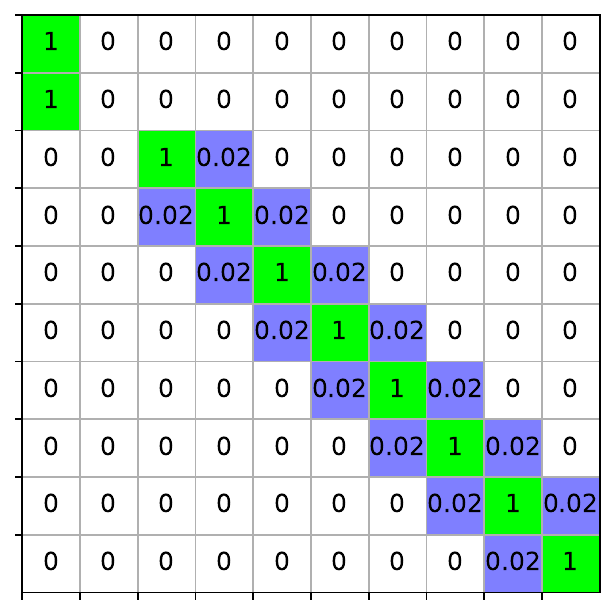}
    \caption{Matrix game, dimension = 10}
\end{figure}
\begin{figure}[H]
    \centering
    \includegraphics[width=\columnwidth]{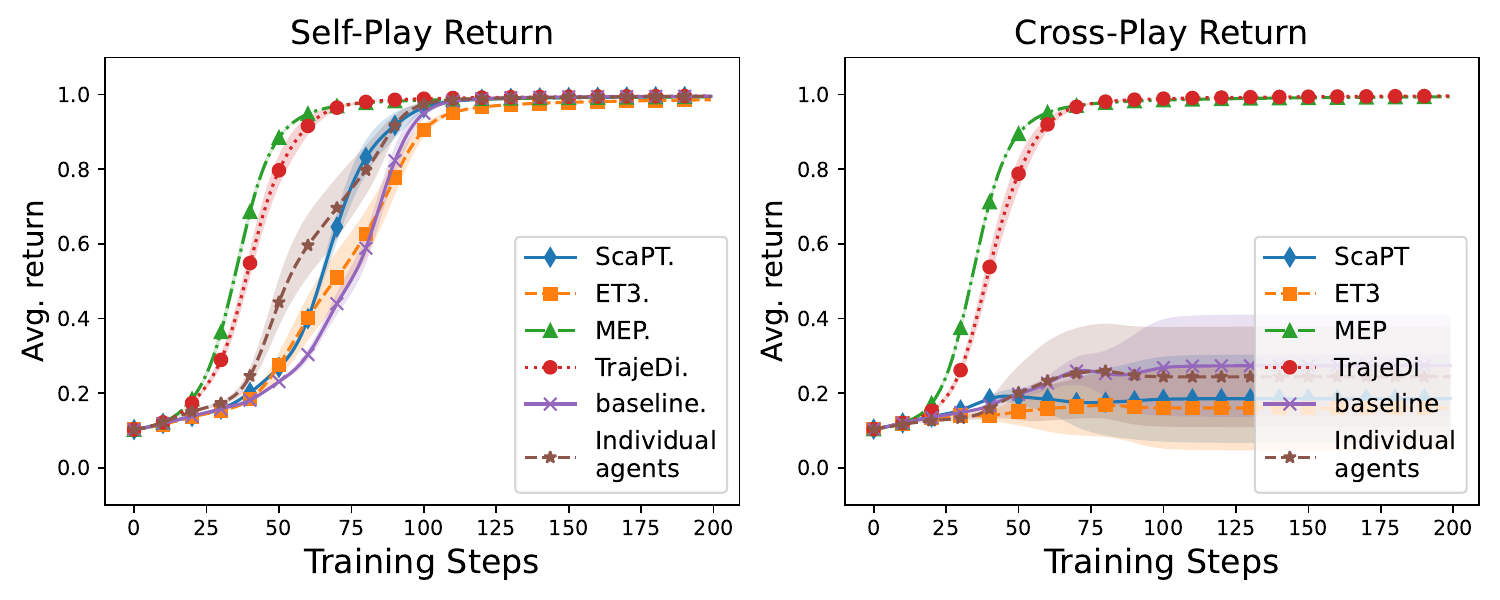}
    \caption{Matrix dimension = 10, population size = 2}
\end{figure}
\begin{figure}[htbp]
    \centering
    \includegraphics[width=\columnwidth]{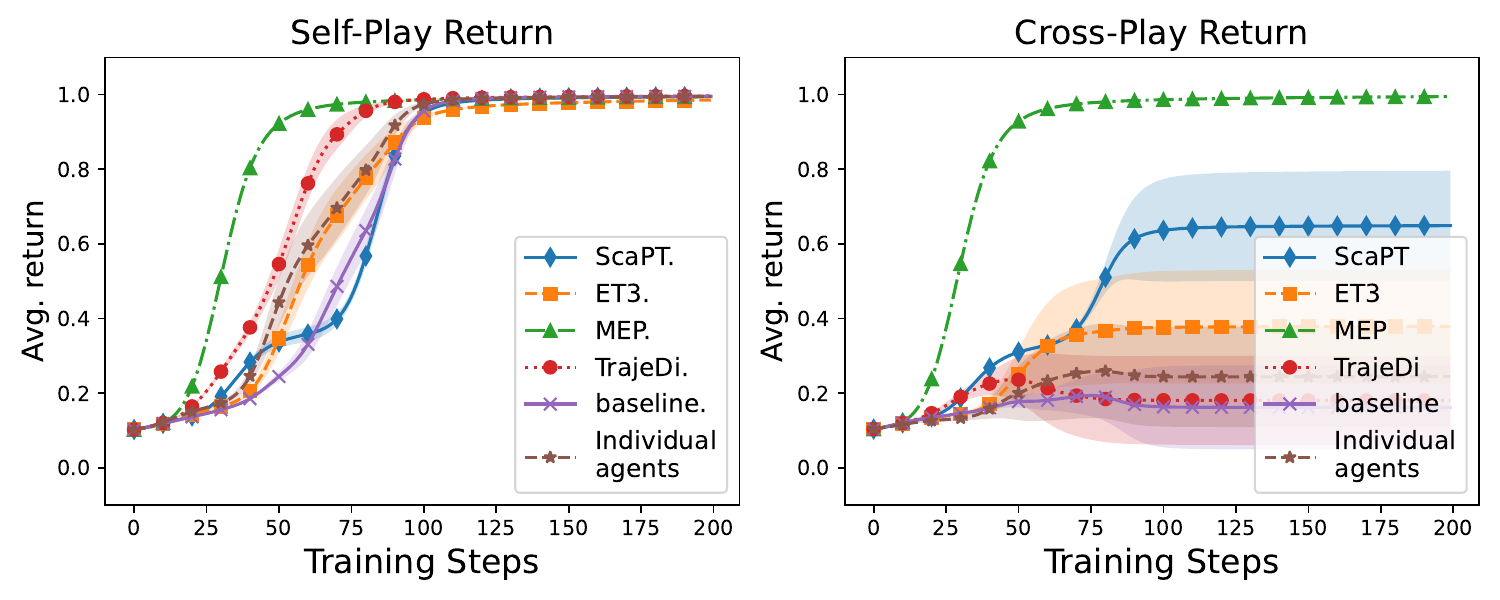}
    \caption{Matrix dimension = 10, population size = 5}
\end{figure}
\begin{figure}[H]
    \centering
    \includegraphics[width=\columnwidth]{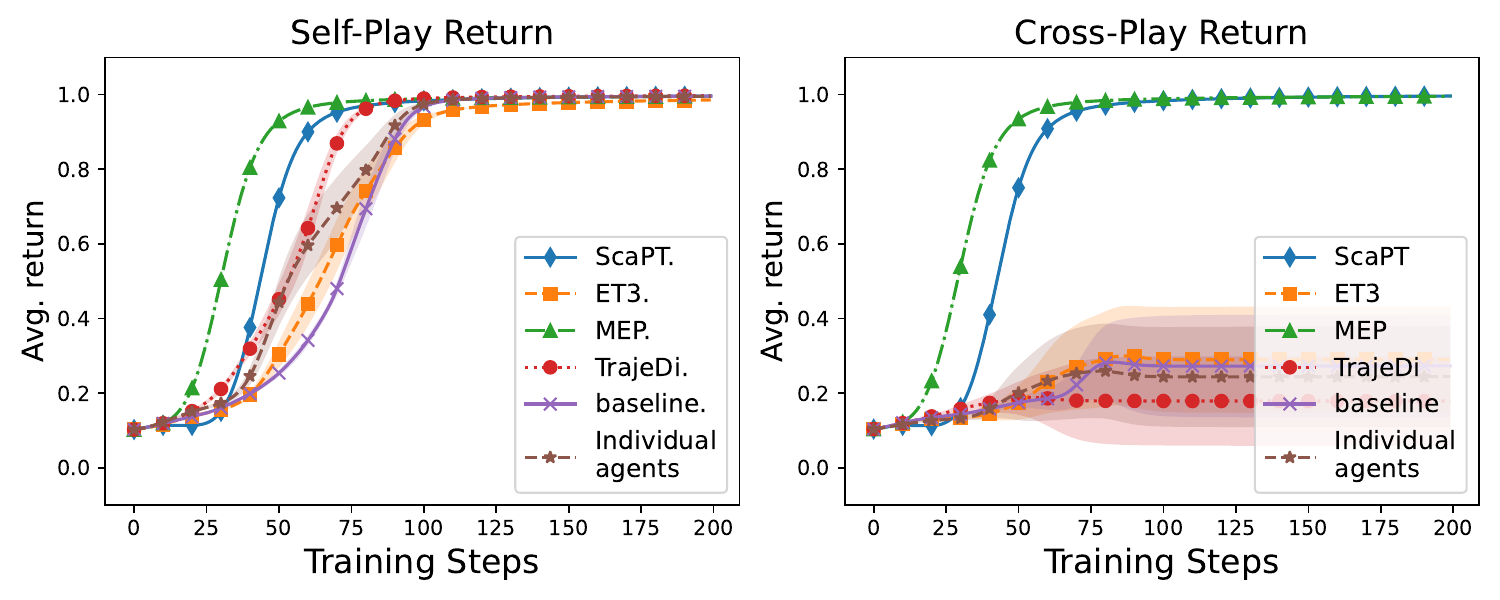}
    \caption{Matrix dimension = 10, population size = 10}
\end{figure}
\FloatBarrier

\subsection{matrix dimension = 30}
\begin{figure}[H]
    \centering
    \includegraphics[width=0.8\columnwidth]{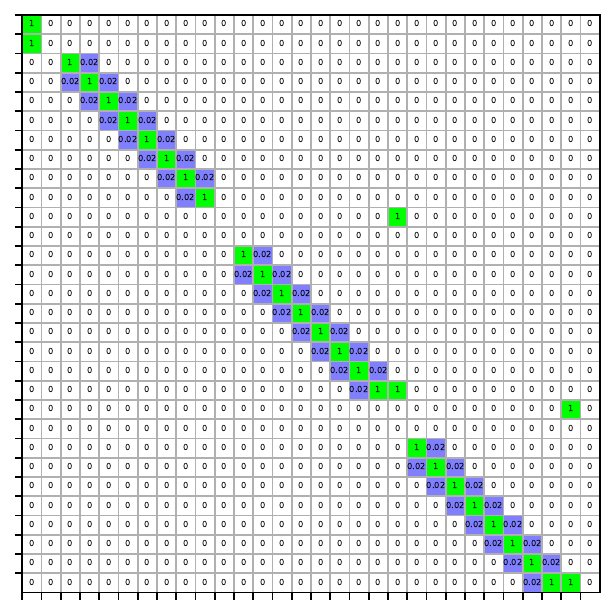}
    \caption{Matrix game, dimension = 10}
\end{figure}
\begin{figure}[H]
    \centering
    \includegraphics[width=\columnwidth]{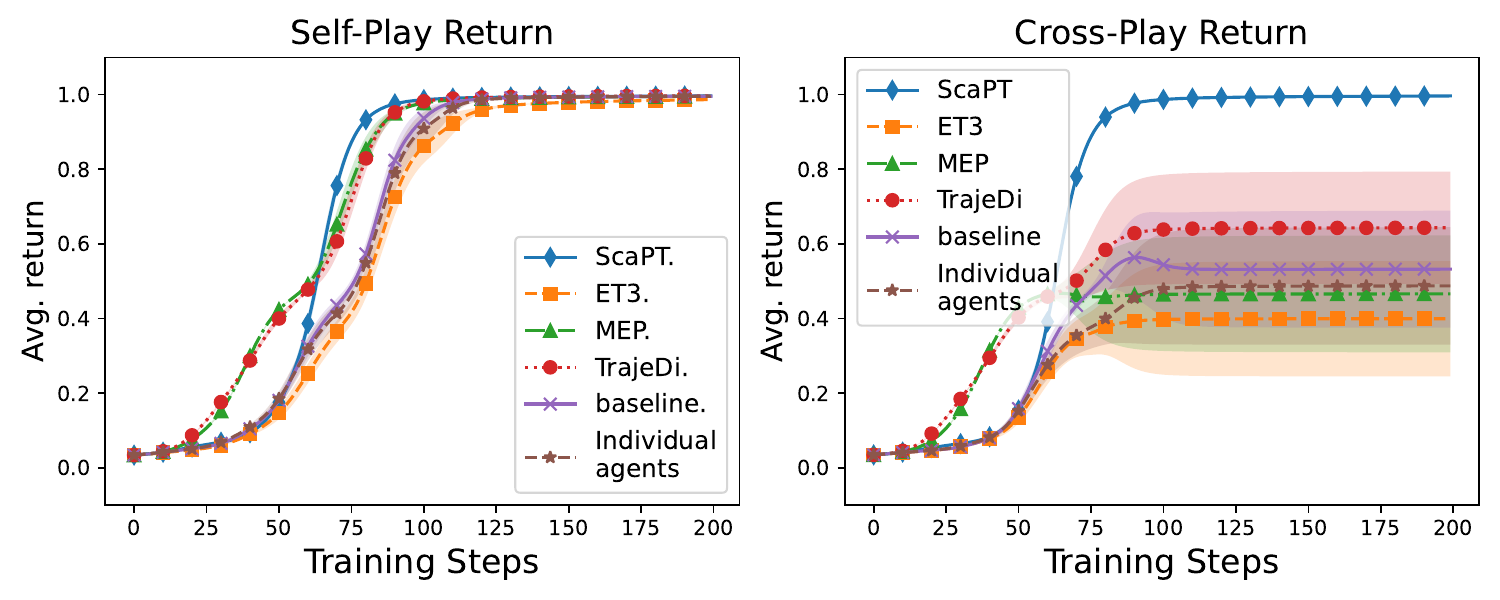}
    \caption{Matrix dimension = 30, population size = 2}
\end{figure}

\begin{figure}[H]
    \centering
    \includegraphics[width=\columnwidth]{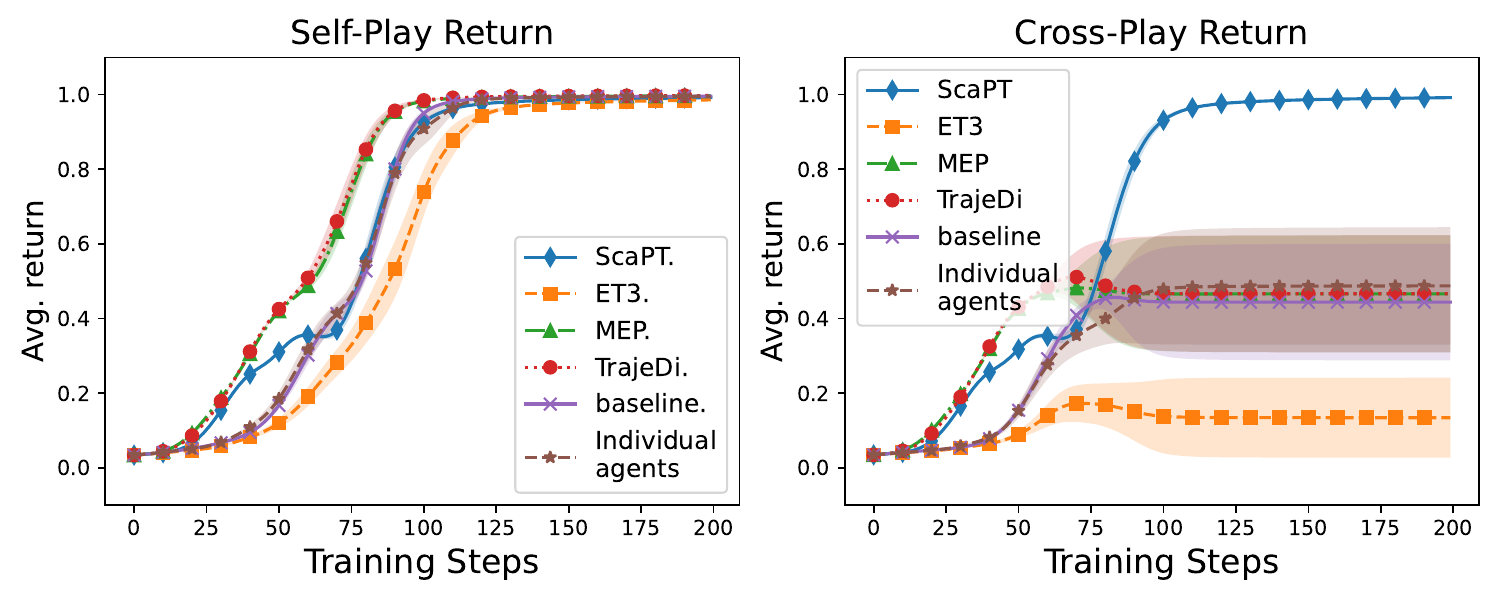}
    \caption{Matrix dimension = 30, population size = 10}
\end{figure}

\begin{figure}[H]
    \centering
    \includegraphics[width=\columnwidth]{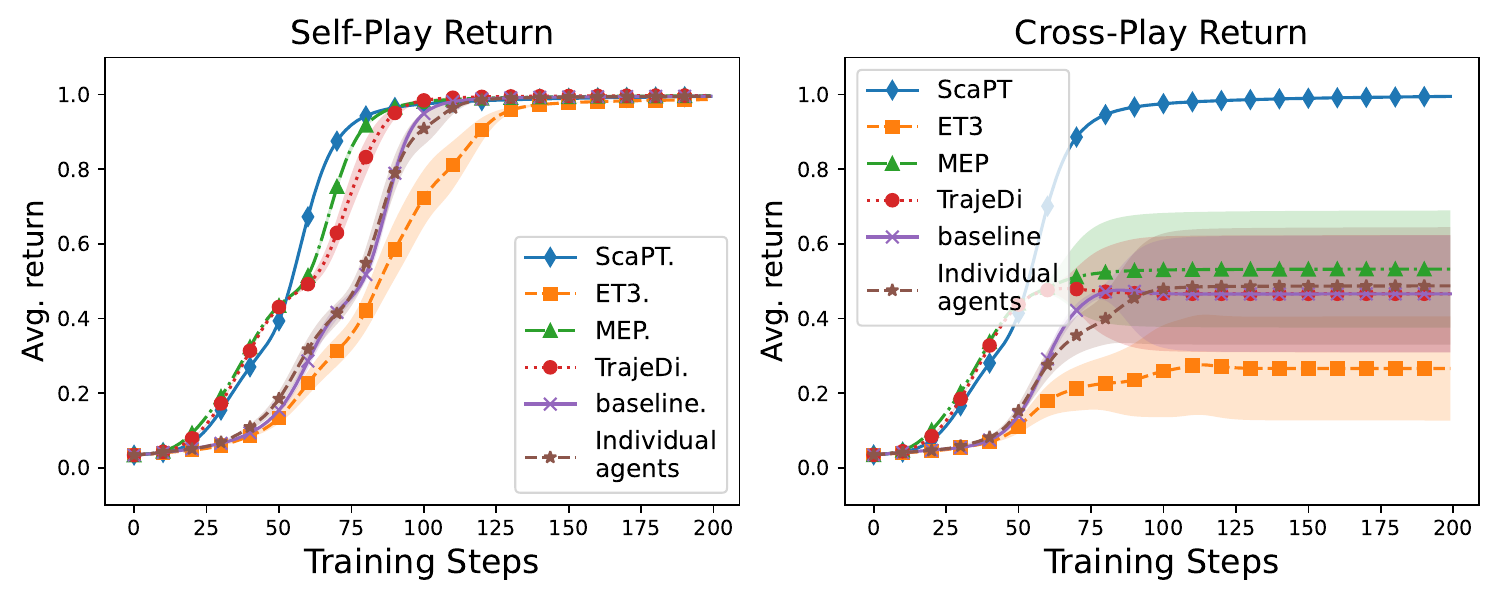}
    \caption{Matrix dimension = 30, population size = 20}
\end{figure}

\subsection{matrix dimension = 50}
\begin{figure}[H]
    \centering
    \includegraphics[width=0.8\columnwidth]{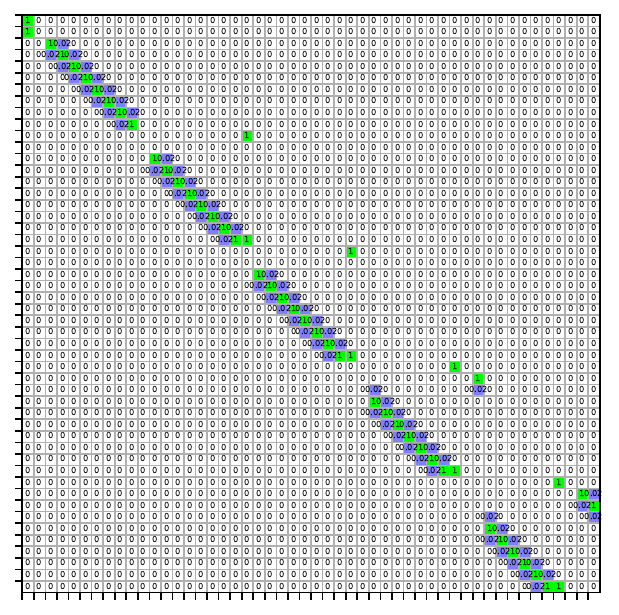}
    \caption{Matrix game, dimension = 50}
\end{figure}
\begin{figure}[H]
    \centering
    \includegraphics[width=\columnwidth]{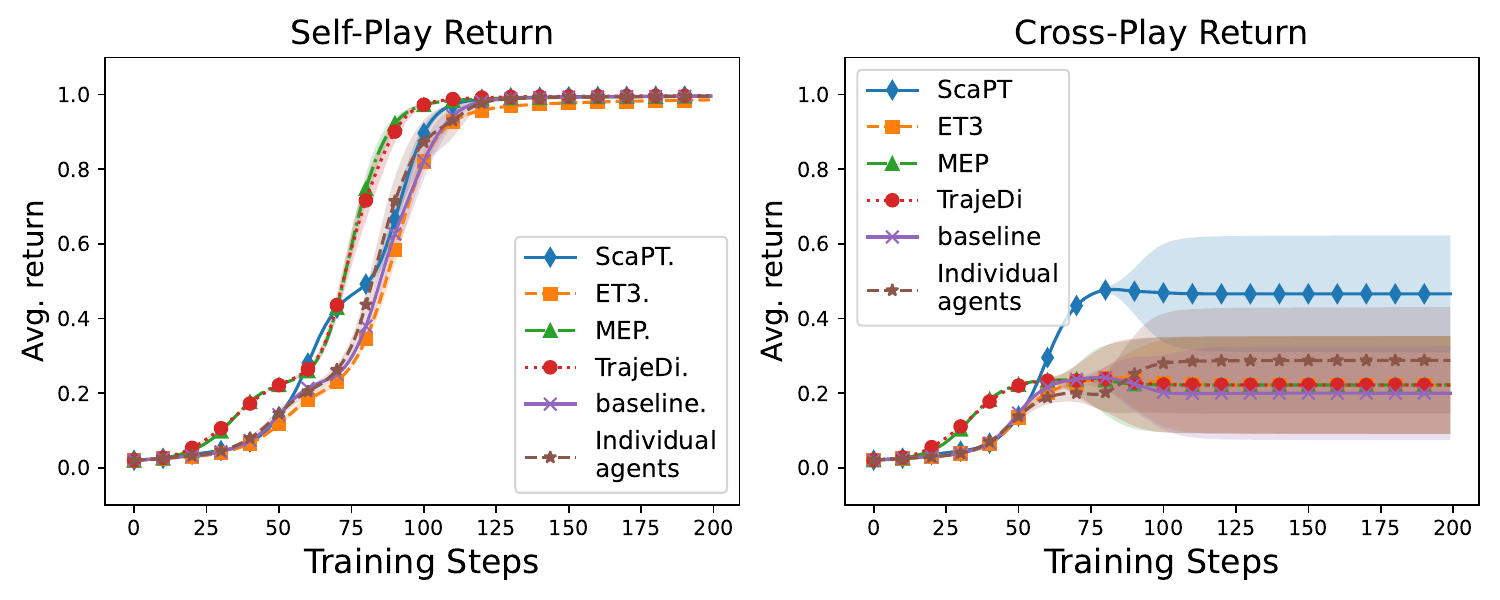}
    \caption{Matrix dimension = 50, population size = 2}
\end{figure}

\begin{figure}[H]
    \centering
    \includegraphics[width=\columnwidth]{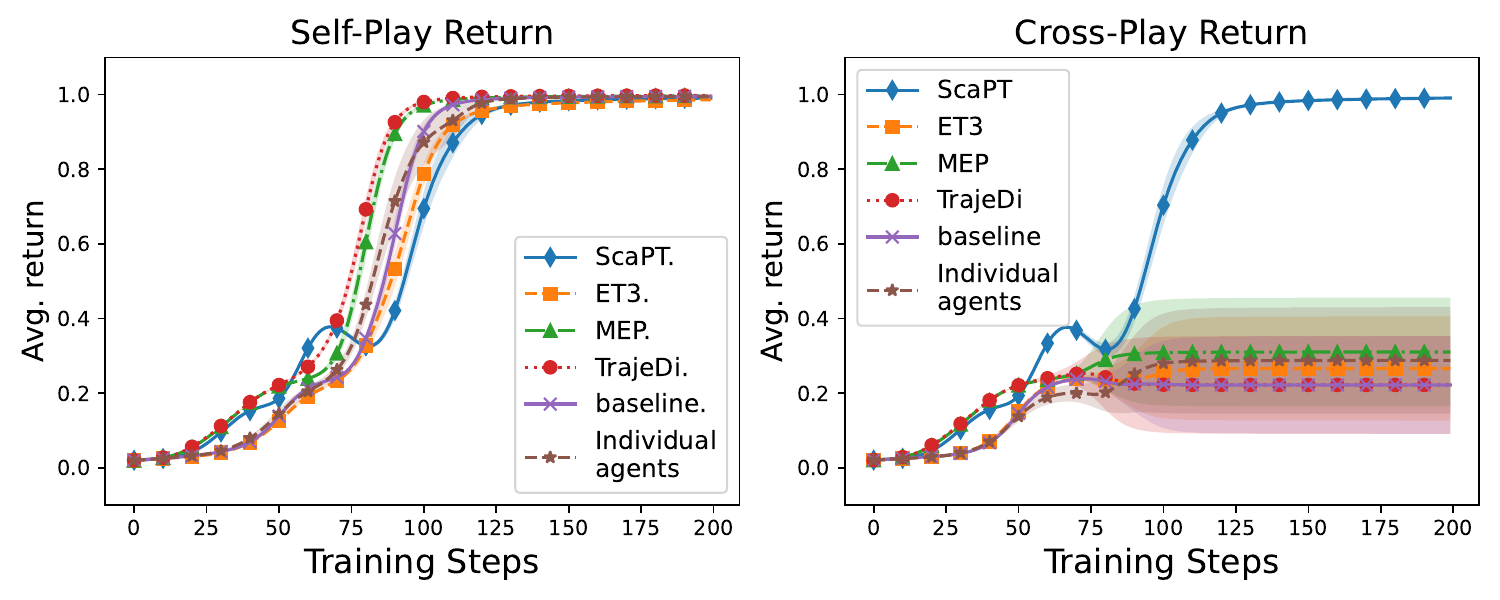}
    \caption{Matrix dimension = 50, population size = 10}
\end{figure}

\begin{figure}[H]
    \centering
    \includegraphics[width=\columnwidth]{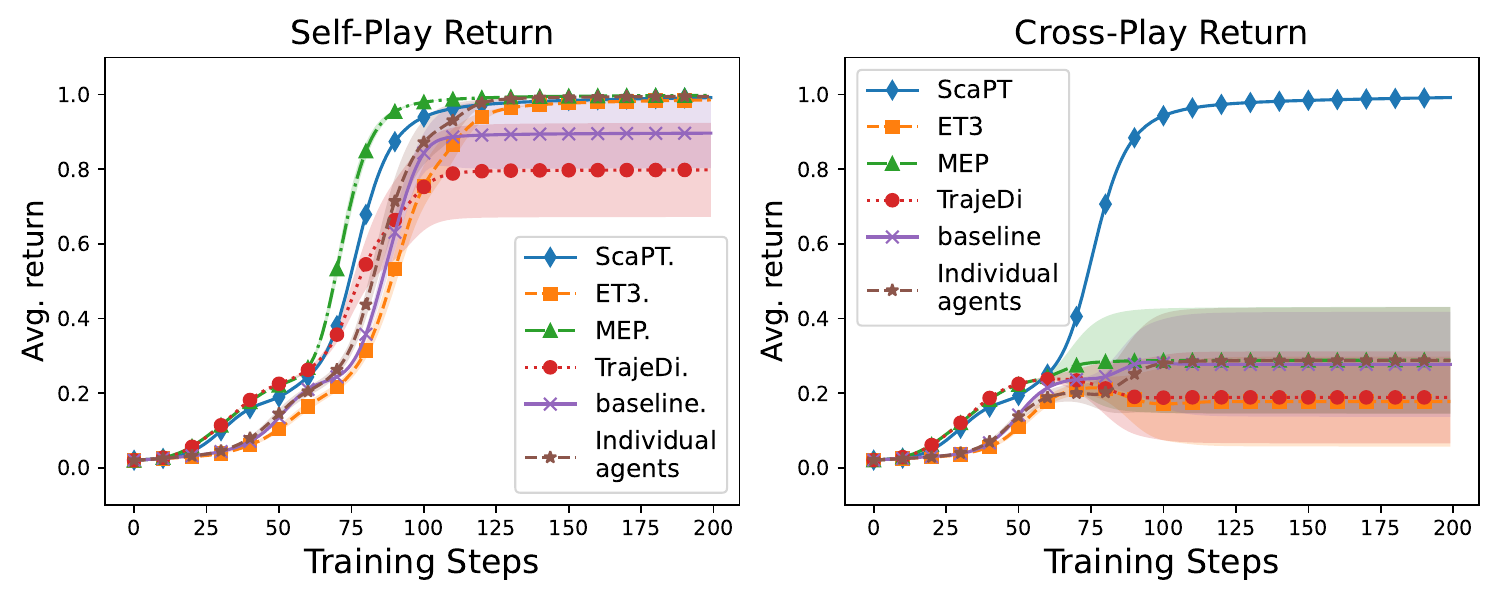}
    \caption{Matrix dimension = 50, population size = 20}
\end{figure}

\begin{figure}[H]
    \centering
    \includegraphics[width=\columnwidth]{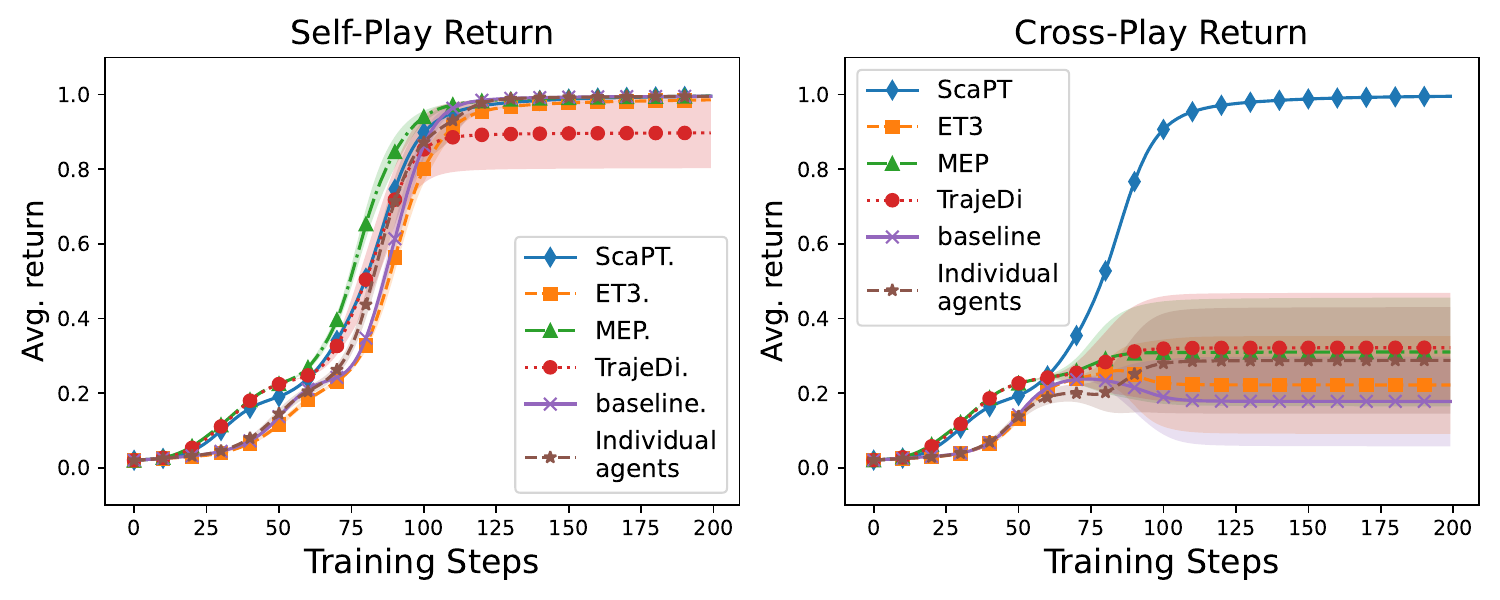}
    \caption{Matrix dimension = 50, population size = 30}
\end{figure}

\subsection{matrix dimension = 100}
\begin{figure}[H]
    \centering
    \includegraphics[width=0.8\columnwidth]{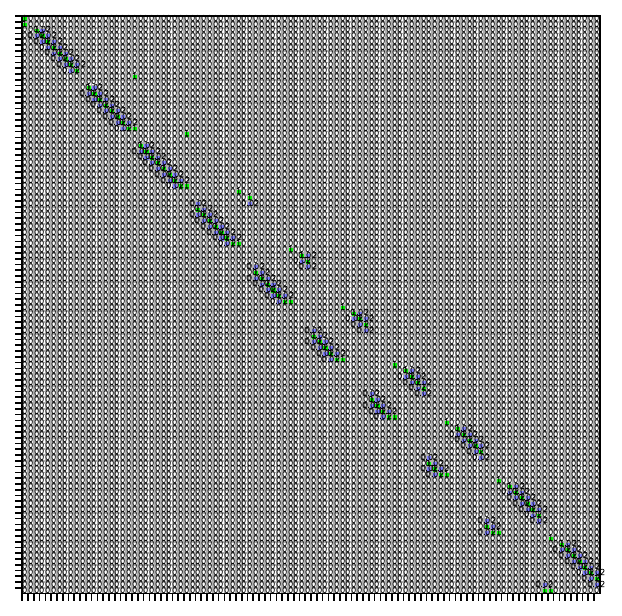}
    \caption{Matrix game, dimension = 100}
\end{figure}
\begin{figure}[H]
    \centering
    \includegraphics[width=\columnwidth]{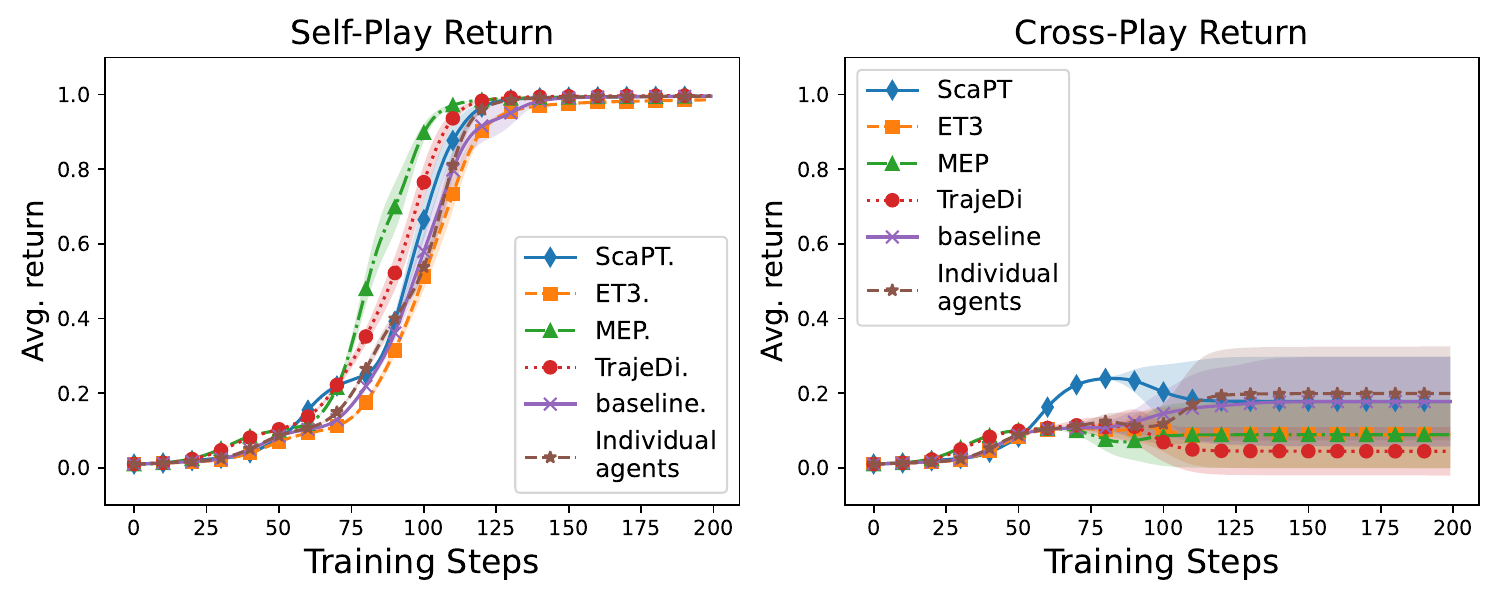}
    \caption{Matrix dimension = 100, population size = 2}
\end{figure}

\begin{figure}[H]
    \centering
    \includegraphics[width=\columnwidth]{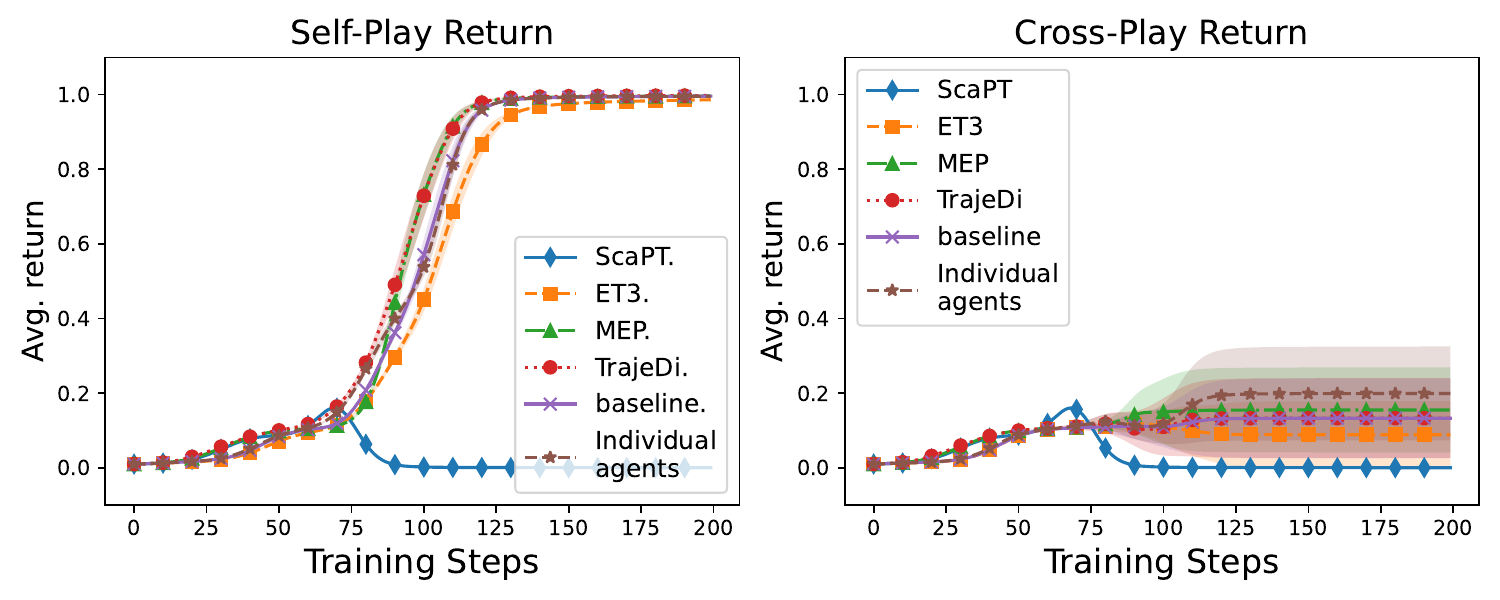}
    \caption{Matrix dimension = 100, population size = 10}
\end{figure}

\begin{figure}[htbp]
    \centering
    \includegraphics[width=\columnwidth]{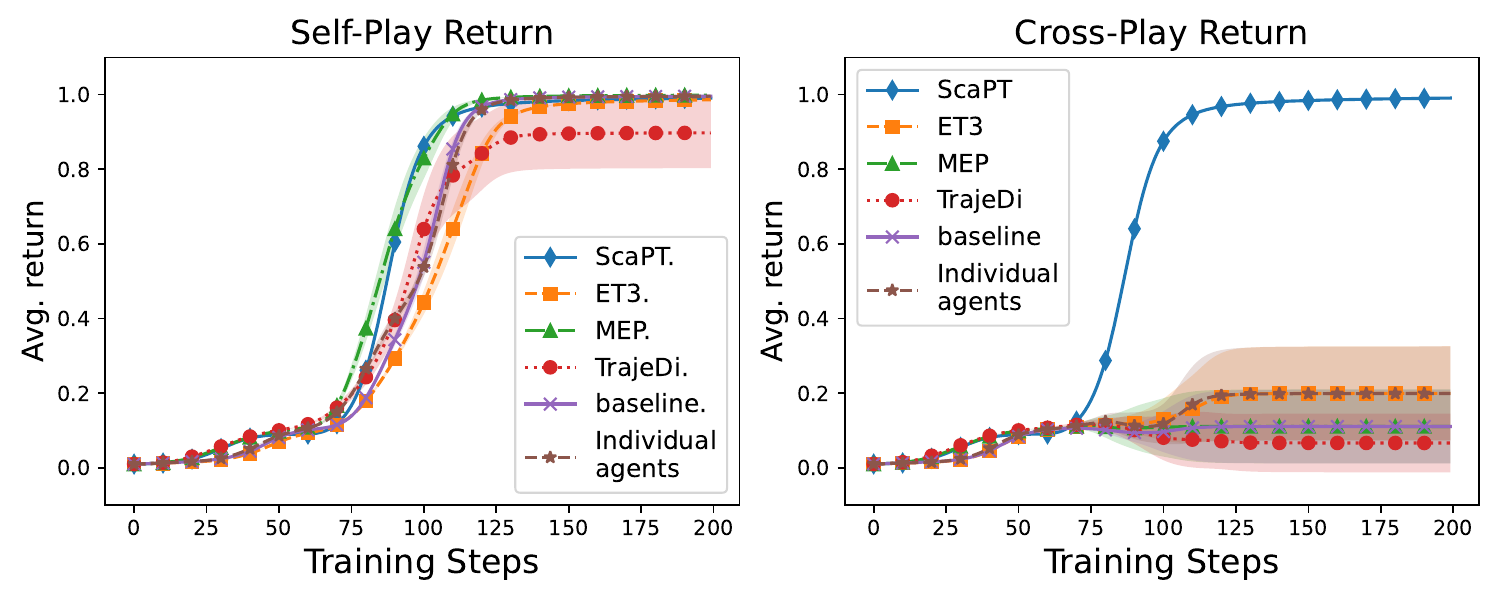}
    \caption{Matrix dimension = 100, population size = 20}
\end{figure}

\begin{figure}[htbp]
    \centering
    \includegraphics[width=\columnwidth]{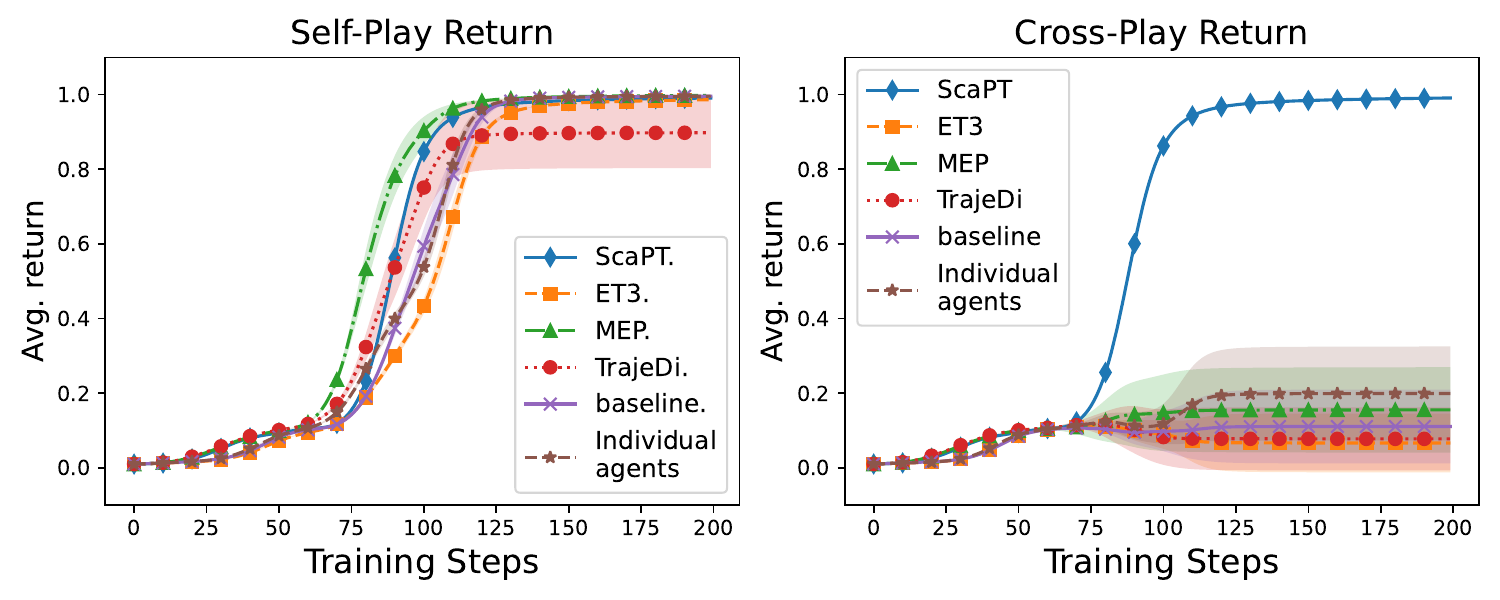}
    \caption{Matrix dimension = 100, population size = 30}
\end{figure}

Based on the aforementioned experiments, we summarize the following findings:
\begin{itemize}
\item As the dimension of the matrix game increases, corresponding to higher task difficulty, the performance of all algorithms degrades or even fails under fixed population sizes.
    \item Among the baselines, population-based methods cannot mitigate this performance degradation by simply increasing population size.
    \item In contrast, our proposed method successfully leverages larger populations to achieve superior performance on more challenging tasks, significantly outperforming other algorithms.
\end{itemize}

\section{D. Detailed results of hanabi game}

In this section, we present the detailed performance of various algorithms on Hanabi games with different numbers of players. As described in the main text, deeper colors represent higher scores, and each row indicates the coordination scores obtained by testing a main agent paired with 40 non-ZSC agents, resulting in a $5 \times 40$ heatmap. Due to the significantly higher training cost of the 5-player Hanabi setting, we report results based on only 4 random seeds for each algorithm in this setting.

\subsection{2-player hanabi game with population size(ps)=5}
\begin{figure}[H]
    \centering
    \includegraphics[width=1\linewidth]{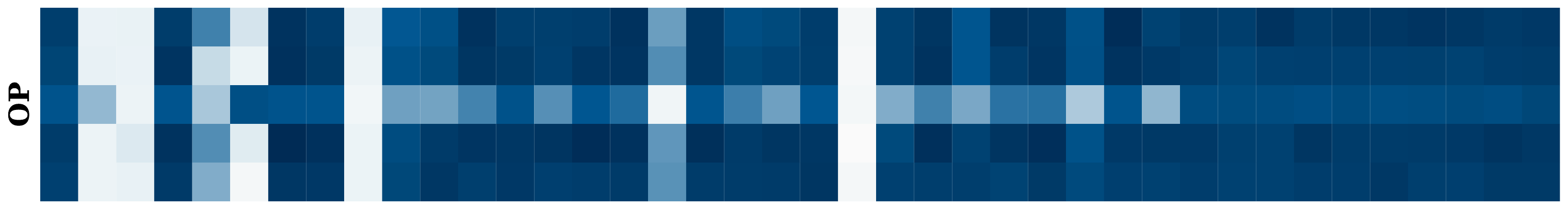}
    \includegraphics[width=1\linewidth]{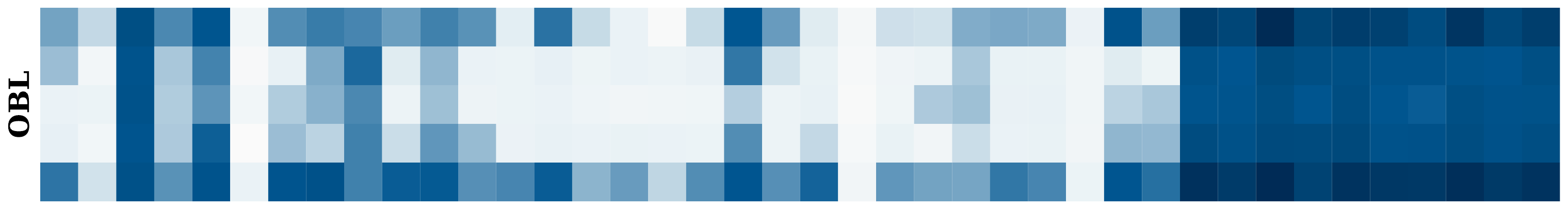}
    \includegraphics[width=1\linewidth]{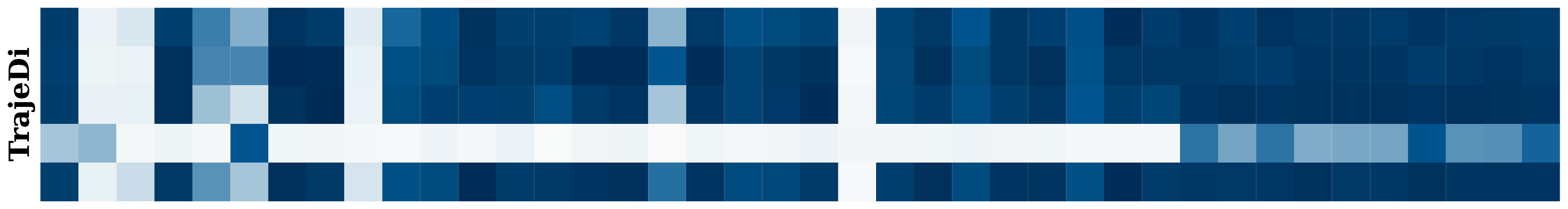}
    \includegraphics[width=1\linewidth]{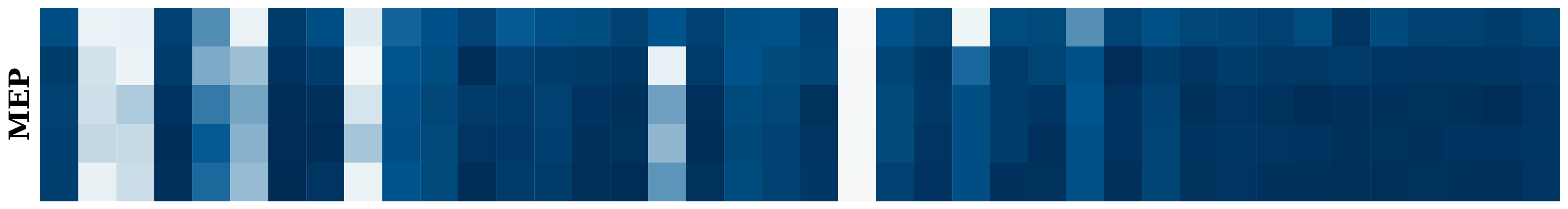}
    \includegraphics[width=1\linewidth]{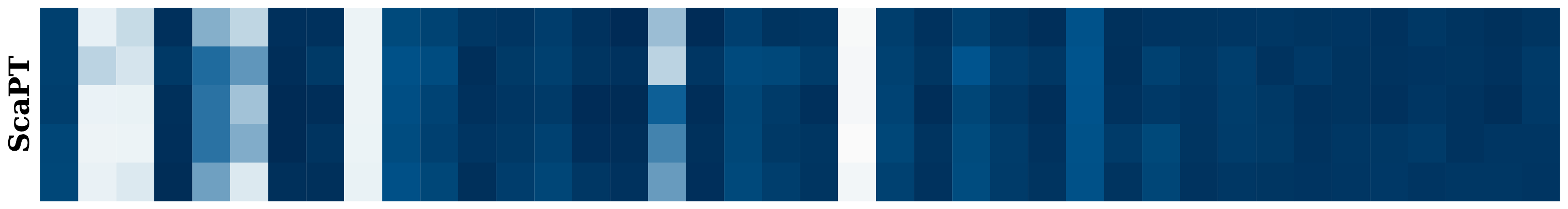}
    \caption{Detailed pair-wise 1ZSC-XP scores of all the testing frameworks(Hanabi, player=2, ps=5).}
    \label{fig: superdetailed}
\end{figure}

\section{Detailed results of 5-player hanabi game with different population-size(ps)}
\begin{figure}[H]  
    \centering

    \begin{subfigure}{\linewidth}
        \centering
        \includegraphics[width=1\linewidth]{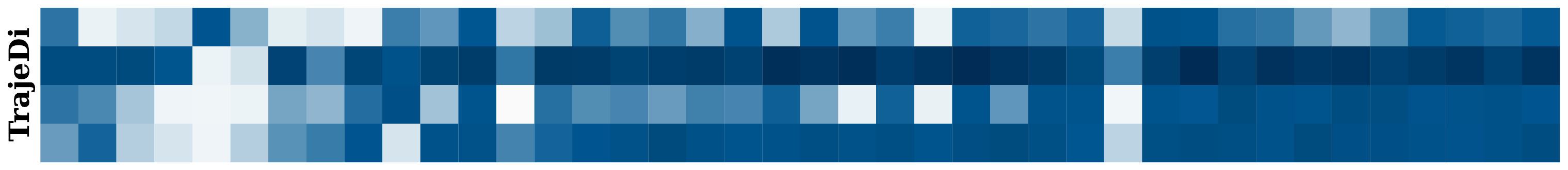}
    \end{subfigure}

    \begin{subfigure}{\linewidth}
        \centering
        \includegraphics[width=1\linewidth]{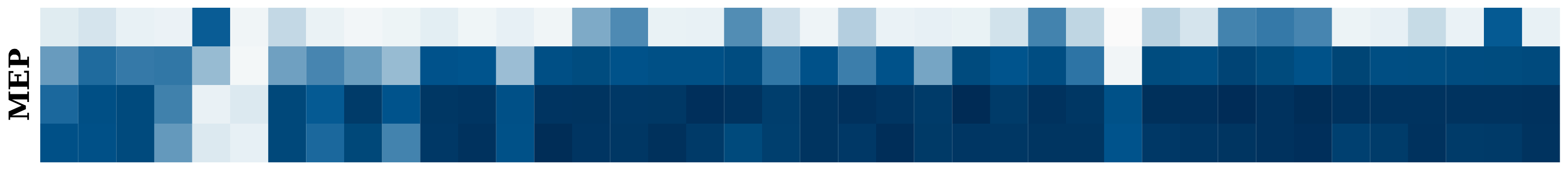}
    \end{subfigure}

    \begin{subfigure}{\linewidth}
        \centering
        \includegraphics[width=1\linewidth]{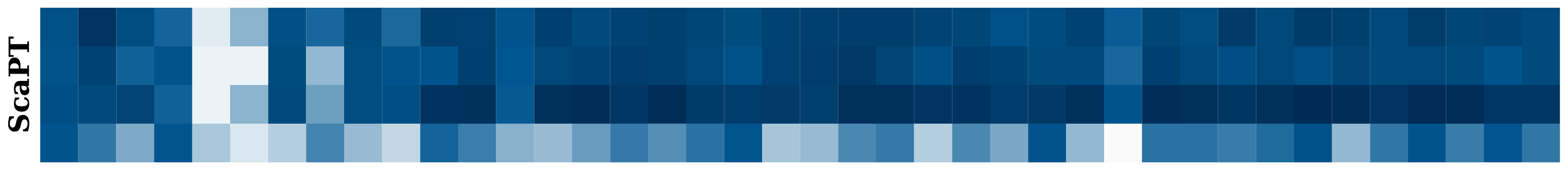}
    \end{subfigure}
    \caption{Detailed pair-wise 1ZSC-XP scores of all the testing frameworks(Hanabi, player=5, ps=2).}
    \label{fig:vertical-train}
\end{figure}

\begin{figure}[H]  
    \centering

    \begin{subfigure}{\linewidth}
        \centering
        \includegraphics[width=1\linewidth]{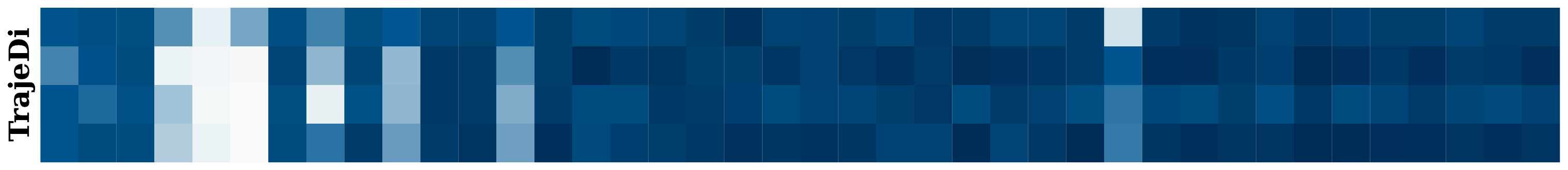}
    \end{subfigure}

    \begin{subfigure}{\linewidth}
        \centering
        \includegraphics[width=1\linewidth]{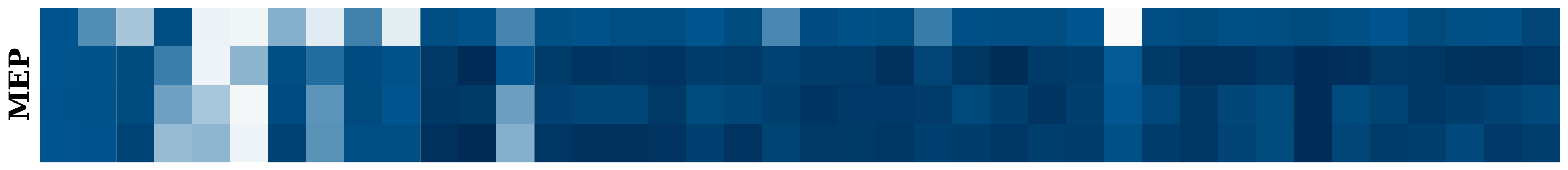}
    \end{subfigure}

    \begin{subfigure}{\linewidth}
        \centering
        \includegraphics[width=1\linewidth]{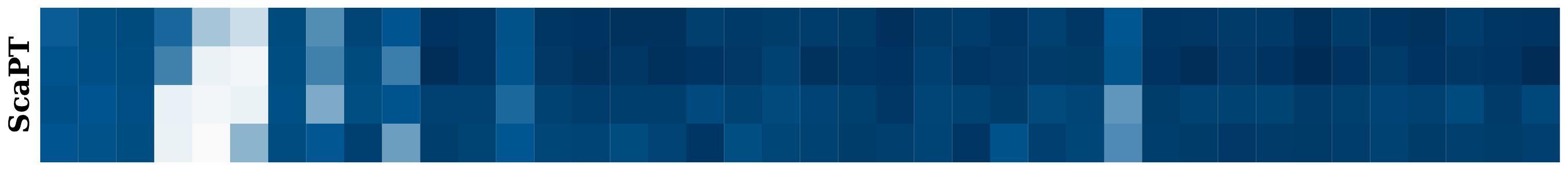}
    \end{subfigure}
    \caption{Detailed pair-wise 1ZSC-XP scores of all the testing frameworks(Hanabi, player=5, ps=5).}
    \label{fig:vertical-train}
\end{figure}

\begin{figure}[H]  
    \centering

    \begin{subfigure}{\linewidth}
        \centering
        \includegraphics[width=1\linewidth]{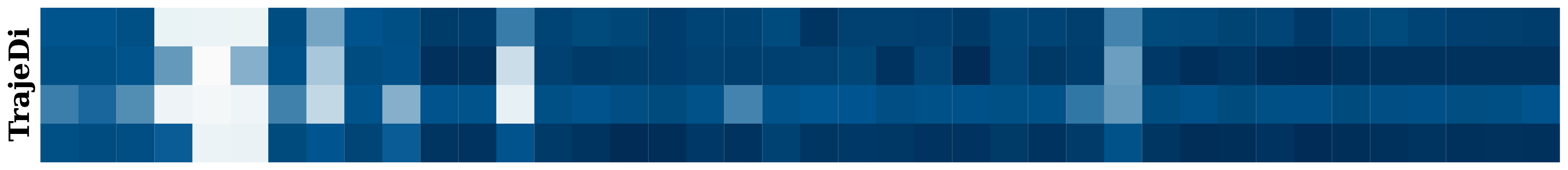}
    \end{subfigure}

    \begin{subfigure}{\linewidth}
        \centering
        \includegraphics[width=1\linewidth]{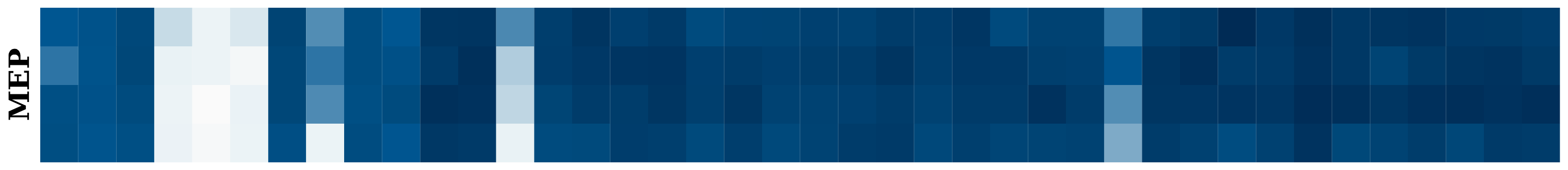}
    \end{subfigure}

    \begin{subfigure}{\linewidth}
        \centering
        \includegraphics[width=1\linewidth]{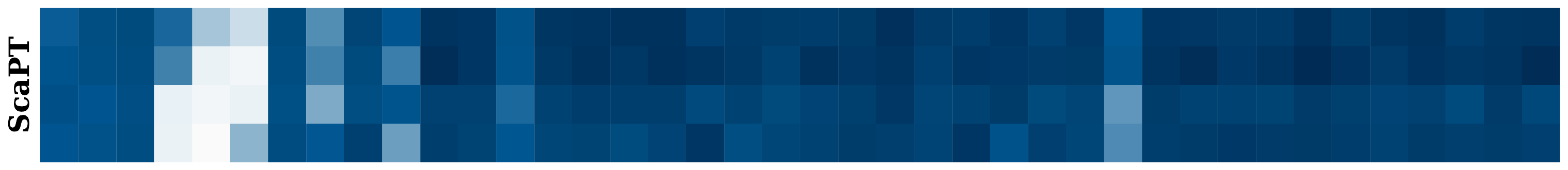}
    \end{subfigure}
    \caption{Detailed pair-wise 1ZSC-XP scores of all the testing frameworks(Hanabi, player=5, ps=8).}
    \label{fig:vertical-train}
\end{figure}

\begin{figure}[h]  
    \centering

    \begin{subfigure}{\linewidth}
        \centering
        \includegraphics[width=1\linewidth]{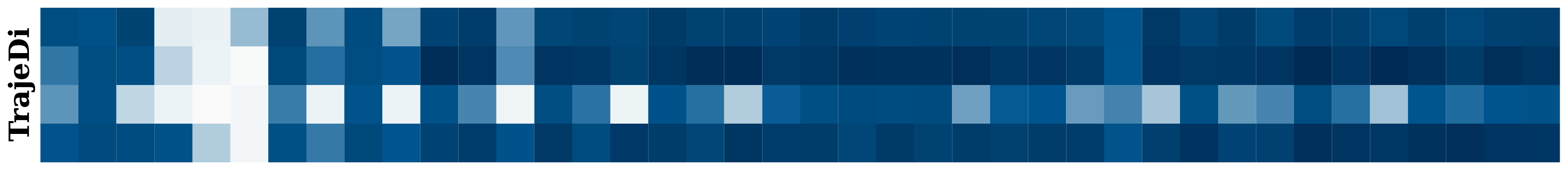}
    \end{subfigure}

    \begin{subfigure}{\linewidth}
        \centering
        \includegraphics[width=1\linewidth]{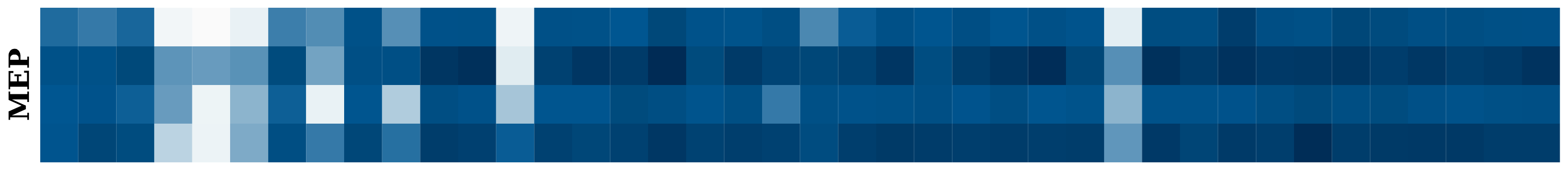}
    \end{subfigure}

    \begin{subfigure}{\linewidth}
        \centering
        \includegraphics[width=1\linewidth]{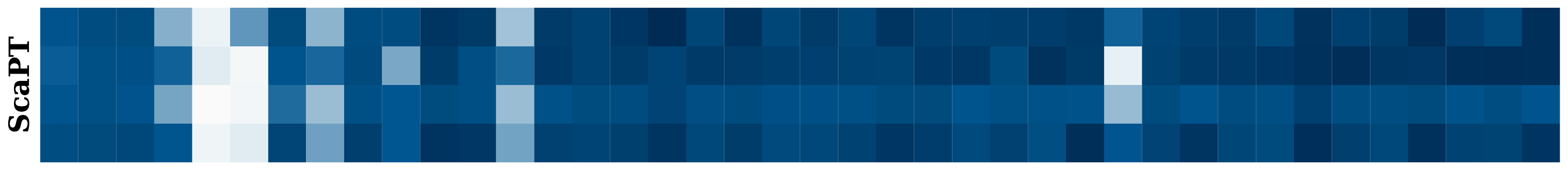}
    \end{subfigure}
    \caption{Detailed pair-wise 1ZSC-XP scores of all the testing frameworks(Hanabi, player=5, ps=15).}
    \label{fig:vertical-train}
\end{figure}

\section{E. Detailed introduction of different training modes}
Table.1 in the main text presents several population training modes, and the following content takes Mode IV as an example to introduce the meaning of each column:
\begin{itemize}
    \item \textbf{Act Group: MM, MP, PP.}: There are three kinds of act groups will be used for interacting with the environment and generating transitions: [Main agent, Main agent], [Main agent, Partner agent] and [Partner agent, Partner agent]. Consequently, 
    \item \textbf{Optimization Objective for $\pi_m$: $J(\pi_m,\pi_m) + \sum_{i=1}^N J(\pi_m,\pi_{pi})$}: The main agent is required to cooperate well with itself and partner agents. This influences the transitions used for training main agent: it has two kinds of transitions, which are playing records with another main agent and playing records with a partner agent, and both of them are used for calculating main agent loss defined in Equation (6).
    \item \textbf{Optimization Objective for $\pi_p$: $J(\pi_{pi},\pi_{pi})$}: The partner agents are only required to cooperate well with itself. Notably, it has two kinds of transitions, which are playing records with another partner agent and playing records with a main agent, and only the first will be used for calculating partner agent loss defined in Equation (7). In contrast, in Mode VI, two kinds of transitions are both used for training due to the different optimization objective for $\pi_p$.
\end{itemize}

\noindent
\begin{minipage}{\columnwidth}
\small
\begin{algorithm}[H]
    \renewcommand{\algorithmicrequire}{\textbf{INPUT:}}
    \renewcommand{\algorithmicensure}{\textbf{OUTPUT:}}
    \caption{Training process of ScaPT with Mode IV} 
    \label{alg:1} 
    \begin{algorithmic}[1]
        \REQUIRE Mutual information term weight $\alpha$, batch size $N_b$, replay buffers $A$, $B$;
        \STATE Initialize $\theta \leftarrow \text{random}$, $\theta_p \leftarrow \text{random}$;
        \STATE Define action groups:
        \STATE \quad $G_1 = [\text{Main agent}, \text{Main agent}]$
        \STATE \quad $G_2 = [\text{Main agent}, \text{Partner agent}]$
        \STATE \quad $G_3 = [\text{Partner agent}, \text{Partner agent}]$
        \WHILE{not reached maximum iterations}
            \FOR{$G \in \{G_1, G_2, G_3\}$}
                \STATE Reset environment if necessary;
                \STATE $o_1^t, o_2^t \leftarrow \text{Observe}(G)$;
                \STATE $h_1^t, h_2^t \leftarrow \text{Update\_hidden\_states}(o_1^t, o_2^t)$;
                \STATE $a_1^t \leftarrow \pi_{\theta}(h_1^t), a_2^t \leftarrow \pi_{\theta}(h_2^t)$;
                \STATE $r_1^t, r_2^t \leftarrow \text{Environment\_rewards}(a_1^t, a_2^t)$;
                
                \IF{$G = G_1$}
                    \STATE Store $(o_1^t, h_1^t, a_1^t, r_1^t, o_1^{t+1})$ and $(o_2^t, h_2^t, a_2^t, r_2^t, o_2^{t+1}) \in A$;
                \ENDIF
                \IF{$G = G_2$}
                    \STATE Store $(o_1^t, h_1^t, a_1^t, r_1^t, o_1^{t+1}) \in A$;
                \ENDIF
                \IF{$G = G_3$}
                    \STATE Store $(o_1^t, h_1^t, a_1^t, r_1^t, o_1^{t+1})$ and $(o_2^t, h_2^t, a_2^t, r_2^t, o_2^{t+1}) \in B$;
                \ENDIF
            \ENDFOR
            
            \STATE Update networks:
            \STATE Sample $N_b$ transitions from $A$, update $\theta$ using loss from Equation (6);
            \STATE Sample $N_b$ transitions from $B$, update $\theta_p$ using loss from Equation (7);
        \ENDWHILE         
    \end{algorithmic} 
\end{algorithm}
\end{minipage}

Algorithm \ref{alg:1} introduces the training process.
\vspace{-2.5em}

\begin{figure}[H]  
    \centering
    \begin{subfigure}[t]{0.48\columnwidth}
        \includegraphics[width=\linewidth]{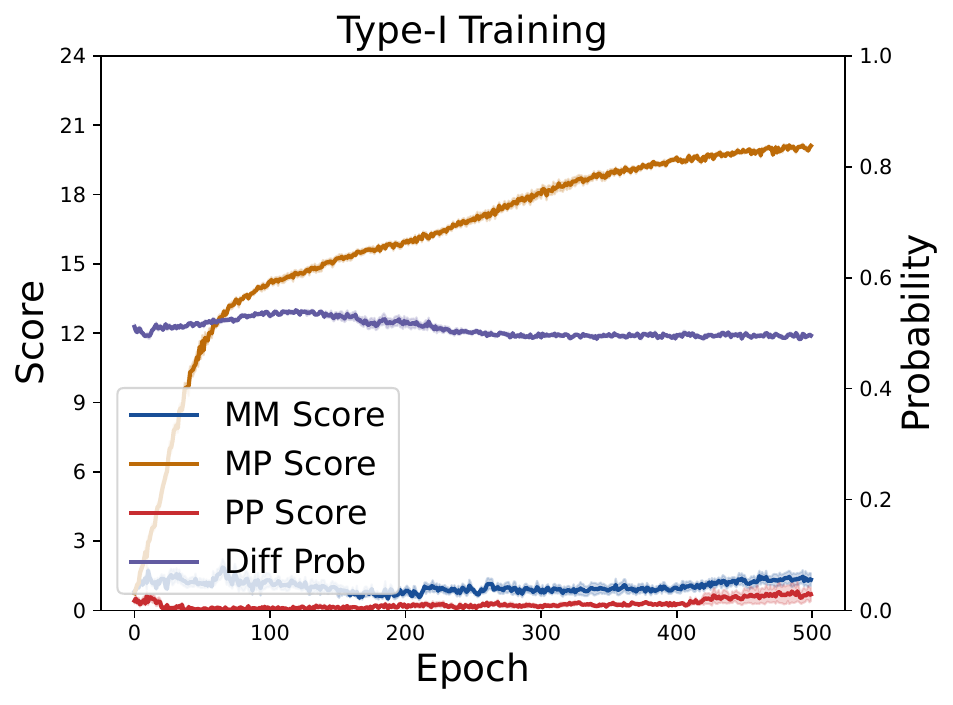}
        \caption{}
    \end{subfigure}
    \hfill
    \begin{subfigure}[t]{0.48\columnwidth}
        \includegraphics[width=\linewidth]{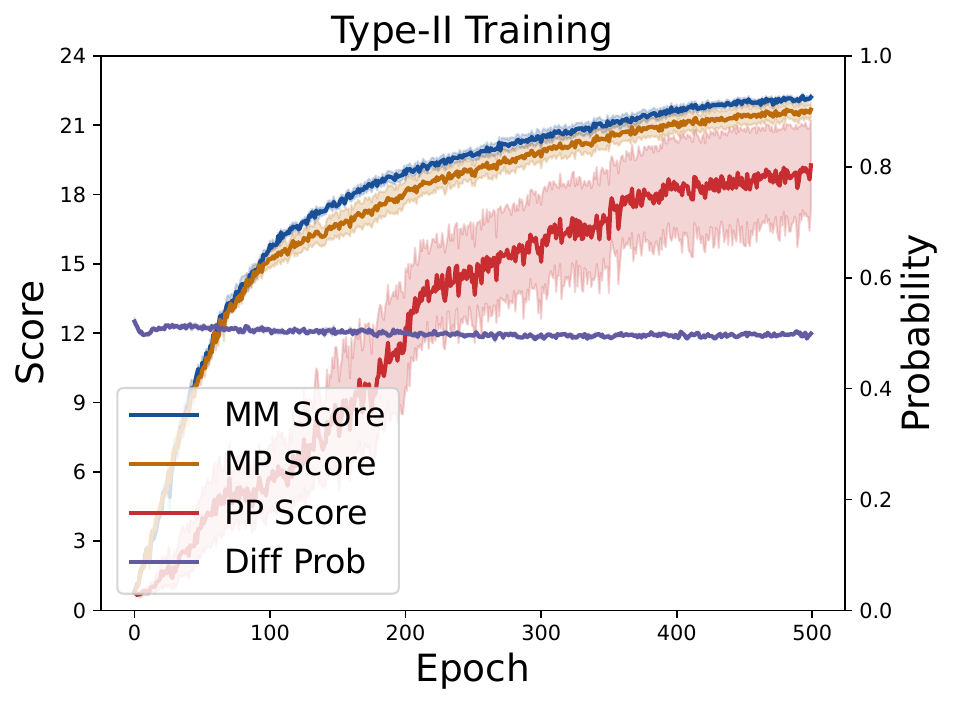}
        \caption{}
    \end{subfigure}

    \begin{subfigure}[t]{0.48\columnwidth}
        \includegraphics[width=\linewidth]{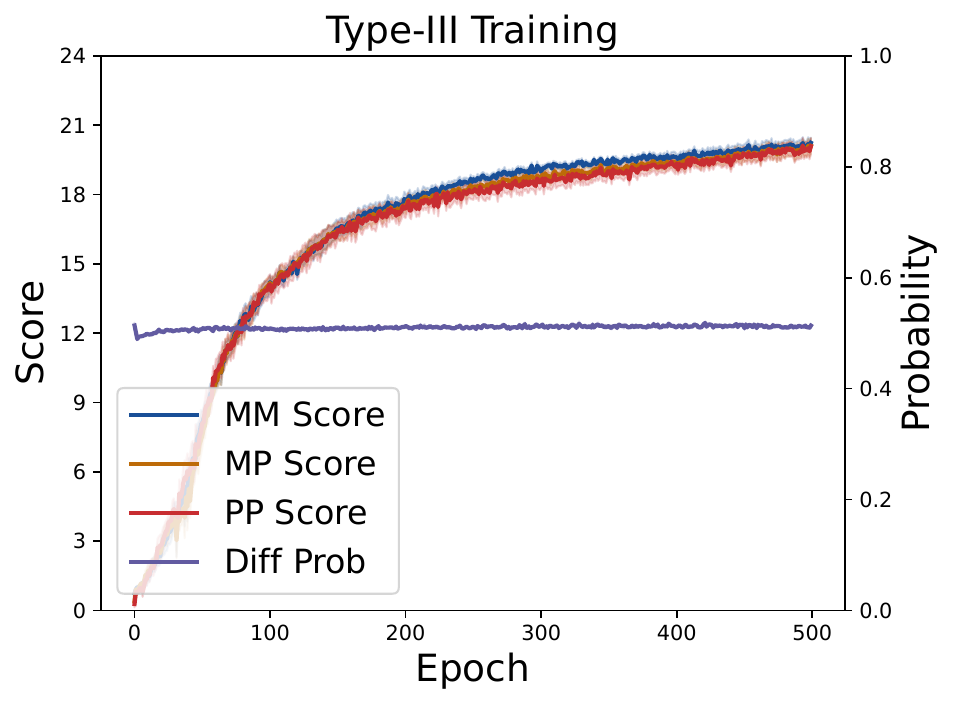}
        \caption{}
    \end{subfigure}
    \hfill
    \begin{subfigure}[t]{0.48\columnwidth}
        \includegraphics[width=\linewidth]{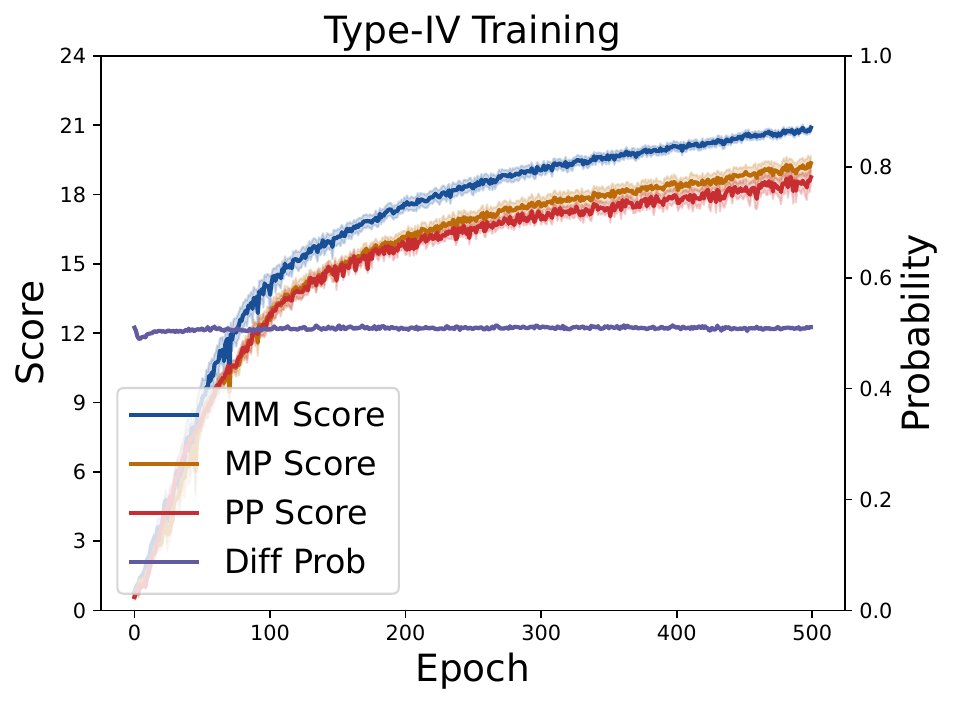}
        \caption{}
    \end{subfigure}

\end{figure}

\begin{figure}[H]\ContinuedFloat
    \centering
    \begin{subfigure}[t]{0.48\columnwidth}
        \includegraphics[width=\linewidth]{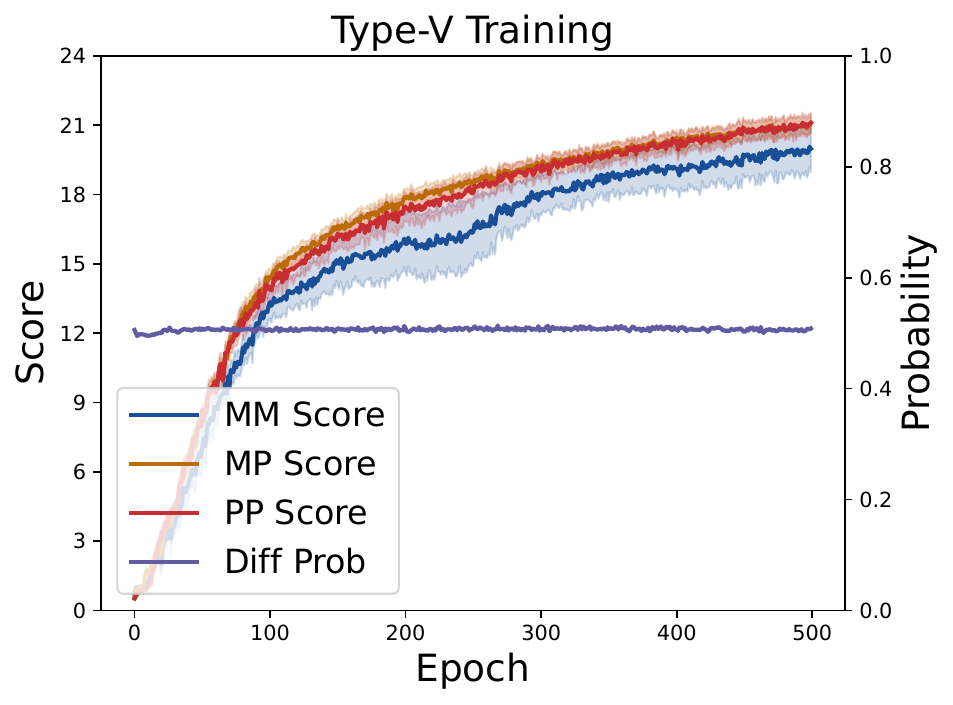}
        \caption{}
    \end{subfigure}
    \hfill
    \begin{subfigure}[t]{0.48\columnwidth}
        \includegraphics[width=\linewidth]{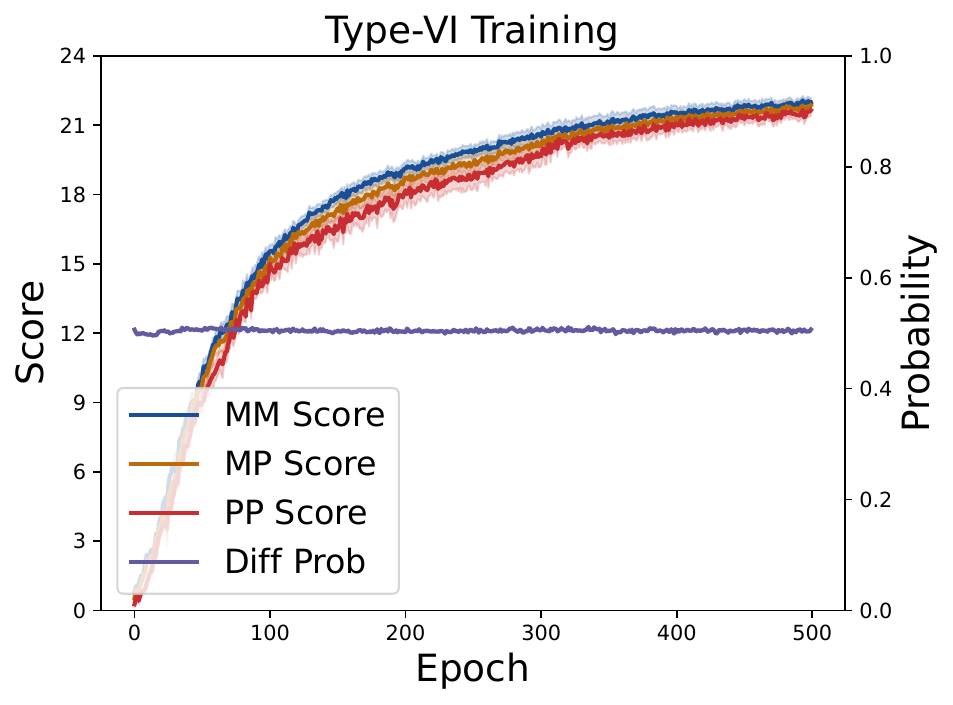}
        \caption{}
    \end{subfigure}
    \caption{Training curves of different training modes.}
    \label{fig:type_all}
\end{figure}




Fig.24 shows the training curves of the different training modes. It can be seen that \textbf{Diff Prob} and \textbf{MP Score} exhibit consistent trends across all training modes, indicating that different modes are consistent in optimizing the primary objective ($J(\pi_m,\pi_p)$) and enhancing diversity.

\section*{Acknowledgments}
This work is supported by National Key Research and Development Program of China (under Grant No.2023YFB2903904), and BNRist projects (No.BNR20231880004 and No.BNR2024TD03003).

\bibliography{ScaPT}

\end{document}